%% file: LinkageExperiments.tex
\documentclass{article}

\title{Learning to Link}

\author{%
Maria-Florina Balcan\\Carnegie Mellon University\\\texttt{ninamf@cs.cmu.edu}%
\and Travis Dick\\Toyota Technical Institute at Chicago\\\texttt{tdick@ttic.edu}%
\and Manuel Lang\\Karlsruhe Institute of Technology\\\texttt{manuel.lang@student.kit.edu}}

\usepackage{fullpage}
\usepackage[utf8]{inputenc} 
\usepackage[T1]{fontenc}    
\usepackage{hyperref}       
\usepackage{url}            
\usepackage{booktabs}       
\usepackage{amsfonts}       
\usepackage{nicefrac}       
\usepackage{microtype}      

\usepackage{mathtools, amsthm, amsfonts}
\usepackage{thmtools, thm-restate}
\usepackage{enumitem}
\usepackage{algorithm}
\usepackage{xcolor}
\usepackage{graphicx}
\usepackage{subcaption}
\usepackage{hyperref}
\usepackage{bm}
\usepackage{cleveref}
\usepackage{natbib}
\usepackage{pgfplots}

\newcommand \cX {\mathcal{X}}
\newcommand \cA {\mathcal{A}}
\newcommand \mergeFamily {\cA_\text{merge}}
\newcommand \mergeAlg [1]{A_{#1}^\text{merge}}
\newcommand \metricFamily {\cA_\text{metric}}
\newcommand \metricAlg [1]{A_{#1}^\text{metric}}

\newcommand \reals {\mathbb{R}}

\newcommand \argmin {\operatorname*{argmin}}

\newcommand \alo {\alpha_\text{lo}}
\newcommand \ahi {\alpha_\text{hi}}
\newcommand \blo {\beta_\text{lo}}
\newcommand \bhi {\beta_\text{hi}}

\newcommand \clusters {\mathcal{N}}

\newcommand \metricp [1]{\operatorname{d_{#1}}}
\newcommand \dbeta {\metricp{\beta}}
\newcommand \dOne {\metricp{1}}
\newcommand \dZero {\metricp{0}}

\newcommand \cmetric [1]{\operatorname{D_{#1}}}
\newcommand \mergep [1]{\cmetric{#1}}
\newcommand \Dalpha {\cmetric{\alpha}}
\newcommand \DOne {\cmetric{1}}
\newcommand \DZero {\cmetric{0}}
\newcommand \Dmin {\cmetric{\text{min}}}
\newcommand \Dmax {\cmetric{\text{max}}}
\newcommand \Davg {\cmetric{\text{avg}}}

\newcommand \leaf {\operatorname{Leaf}}
\newcommand \node {\operatorname{Node}}
\newcommand \target {\mathcal{Y}}

\newcommand \numMetrics {L}
\newcommand \numMerges {{L'}}
\newcommand \cH {\mathcal{H}}
\newcommand \cQ {\mathcal{Q}}

\newcommand \cD {\mathcal{D}}
\newcommand \cL {\mathcal{L}}
\newcommand \expect {\operatorname*{\mathbb{E}}}
\renewcommand \vec [1]{\bm{#1}}
\newcommand \sign {\operatorname{sign}}

\newcommand \region {\mathcal{Z}}
\newcommand \pdim {\operatorname{Pdim}}

\theoremstyle{definition}
\newtheorem{defn}{Definition}

\newcommand \itparagraph [1]{\vspace{0.5em}\noindent\textit{#1}}
\renewcommand \paragraph [1]{\vspace{0.5em}\noindent\textbf{#1}}

\begin{document}

\maketitle

\input{abstract}

\section{Introduction} \label{sec:introduction}
\input{introduction}

\section{Learning Clustering Algorithms} \label{sec:jointFamily}
\input{joint-family}

\section{Efficient Algorithm Selection} \label{sec:ermAlgorithms}
\input{erm-algorithms}

\section{Experiments} \label{sec:experiments}
\input{experiments-new}

\section{Conclusion} \label{sec:conclusion}
\input{conclusion}

\section*{Acknowledgements}

This work was supported in part by NSF grants CCF-1535967, IIS-1618714, an
Amazon Research Award,  a Bloomberg Research Grant, a Microsoft Research Faculty
Fellowship, and by the generosity of Eric and Wendy Schmidt by recommendation of
the Schmidt Futures program.

\bibliographystyle{plainnat}
\bibliography{references}

\appendix

\section{Appendix for Learning Clustering Algorithms} \label{app:jointFamily}
\input{joint-family-appendix}

\section{Appendix for Efficient Algorithm Selection} \label{app:ermAlgorithms}
\input{erm-algorithms-appendix}

\section{Appendix for Experiments} \label{app:experiments}
\input{experiments-new-appendix}

\end{document}

%% file: abstract.tex
\begin{abstract}

Clustering is an important part of many modern data analysis pipelines,
including network analysis and data retrieval. There are many different
clustering algorithms developed by various communities, and it is often not
clear which algorithm will give the best performance on a specific clustering
task. Similarly, we often have multiple ways to measure distances between data
points, and the best clustering performance might require a non-trivial
combination of those metrics. In this work, we study data-driven algorithm
selection and metric learning for clustering problems, where the goal is to
simultaneously learn the best algorithm and metric for a specific application.
The family of clustering algorithms we consider is parameterized linkage based
procedures that includes single and complete linkage. The family of distance
functions we learn over are convex combinations of base distance functions. We
design efficient learning algorithms which receive samples from an
application-specific distribution over clustering instances and learn a
near-optimal distance and clustering algorithm from these classes. We also carry
out a comprehensive empirical evaluation of our techniques showing that they can
lead to significantly improved clustering performance on real-world datasets.

\end{abstract}

%% file: introduction.tex
\paragraph{Overview.} Clustering is an important component of modern data
analysis. For example, we might cluster emails as a pre-processing step for spam
detection, or we might cluster individuals in a social network in order to
suggest new connections. There are a myriad of different clustering algorithms,
and it is not always clear what algorithm will give the best performance on a
specific clustering task. Similarly, we often have multiple different ways to
measure distances between data points, and it is not obvious which distance
metric will lead to the best performance. In this work, we study data-driven
algorithm selection and metric learning for clustering problems, where the goal
is to use data to simultaneously learn the best algorithm and metric for a
specific application such as clustering emails or users of a social network. An
application is modeled as a distribution over clustering tasks, we observe an
i.i.d. sample of clustering instances drawn from that distribution, and our goal
is to choose an approximately optimal algorithm from a parameterized family of
algorithms (according to some well-defined loss function). This corresponds to
settings where we repeatedly solve clustering instances (e.g., clustering the
emails that arrive each day) and we want to use historic instances to learn the
best clustering algorithm and metric.

The family of clustering algorithms we learn over consists of parameterized
linkage based procedures and includes single and complete linkage, which are
widely used in practice and optimal in many cases
\citep{Awasthi14:ActiveClustering, Saeed03:Clustering, White10:Phylo,
Awasthi12:PS, Balcan16:PR, Grosswendt15:CL}. The family of distance metrics we
learn over consists of convex combinations of base distance functions. We design
efficient learning algorithms that receive samples from an application-specific
distribution over clustering instances and simultaneously learn both a
near-optimal distance metric and clustering algorithm from these classes. We
contribute to a recent line of work that provides learning-theoretical
\citep{SBD14:MLBook} guarantees for data-driven algorithm configuration
\citep{Gupta17:PAC, Balcan17:Learning, Balcan18:MIP, Balcan18:kmeans,
Balcan19:General}. These papers analyze the intrinsic complexity of
parameterized algorithm families in order to provide sample complexity
guarantees, that is, bounds on the number of sample instances needed in order to
find an approximately optimal algorithm for a given application domain. Our
results build on the work of \citet{Balcan17:Learning}, who studied the problem
of learning the best clustering algorithm from a class of linkage based
procedures, but did not study learning the best metric. In addition to our
sample complexity guarantees, we develop a number of algorithmic tools that
enable learning application specific clustering algorithms and metrics for
realistically large clustering instances. We use our efficient implementations
to conduct comprehensive experiments on clustering domains derived from both
real-world and synthetic datasets. These experiments demonstrate that learning
application-specific algorithms and metrics can lead to significant performance
improvements over standard algorithms and metrics.

\paragraph{Our Results.} We study linkage-based clustering algorithms that take
as input a clustering instance $S$ and output a hierarchical clustering of $S$
represented as a binary \emph{cluster tree}. Each node in the tree represents a
cluster in the data at one level of granularity, with the leaves corresponding
to individual data points and the root node corresponding to the entire dataset.
Each internal node represents a cluster obtained by merging its two children.
Linkage-based clustering algorithms build a cluster tree from the leaves up,
starting with each point belonging to its own cluster and repeatedly merging the
``closest'' pair of clusters until only one remains. The parameters of our
algorithm family control both the metric used to measure pointwise distances, as
well as how the linkage algorithm measures distances between clusters (in terms
of the distances between their points).

This work has two major contributions.
Our first contribution is to provide sample complexity guarantees for learning
effective application-specific distance metrics for use with linkage-based
clustering algorithms.
The key challenge is that, if we fix a clustering algorithm from the family we
study and a single clustering instance $S$, the algorithm output is a piecewise
constant function of our metric family's parameters.
This implies that, unlike many standard learning problems, the loss we want to
minimize is very sensitive to the metric parameters and small perturbations to
the optimal parameters can lead to high loss.
Our main technical insight is that for any clustering instance $S$, we can
partition the parameter space of our metric family into a relatively small
number of regions such that the ordering over pairs of points in $S$ given by
the metric is constant on each region.
The clustering output by all algorithms in the family we study only depends on
the ordering over pairs of points induced by the metric, and therefore their
output is also a piecewise constant function of the metric parameters with not
too many pieces.
%
%
We leverage this structure to bound the intrinsic complexity of the learning
problem, leading to uniform convergence guarantees.
By combining our results with those of \citet{Balcan17:Learning}, we show how to
simultaneously learn both an application-specific metric and linkage algorithm.


Our second main contribution is a comprehensive empirical evaluation of our
proposed methods, enabled by new algorithmic insights for efficiently learning
application-specific algorithms and metrics from sample clustering instances.
For any fixed clustering instance, we show that we can use an \emph{execution
tree} data structure to efficiently construct a coarse partition of the joint
parameter space so that on each region the output clustering is constant.
Roughly speaking, the execution tree compactly describes all possible sequences
of merges the linkage algorithm might make together with the parameter settings
for the algorithm and metric that lead to that merge sequence.
The learning procedure proposed by \citet{Balcan17:Learning} takes a more
combinatorial approach, resulting in partitions of the parameter space that have
many unnecessary regions and increased overall running time.
\citet{Balcan18:MIP} and \citet{Balcan18:kmeans} also use an execution tree
approach for different algorithm families, however their specific approaches to
enumerating the tree are not efficient enough to be used in our setting.
We show that using a depth-first traversal of the execution tree leads to
significantly reduced memory requirements, since in our setting the execution
tree is shallow but very wide.

Using our efficient implementations, we evaluate our learning algorithms on
several real world and synthetic clustering applications.
We learn the best algorithm and metric for clustering applications derived from
MNIST, CIFAR-10, Omniglot, Places2, and a synthetic rings and disks
distribution.
Across these different tasks the optimal clustering algorithm and metric vary
greatly.
Moreover, in most cases we achieve significant improvements in clustering
quality over standard clustering algorithms and metrics.

\paragraph{Related work.}

\citet{Gupta17:PAC} introduced the theoretical framework for analyzing algorithm
configuration problems that we study in this work. They provide sample
complexity guarantees for greedy algorithms for canonical subset selection
problems including the knapsack problem, maximum weight independent set, and
machine scheduling.

Some recent works provide sample complexity guarantees for learning
application-specific clustering algorithms. \citet{Balcan17:Learning} consider
several parameterized families of linkage based clustering algorithms, one of
which is a special case of the family studied in this paper. Their sample
complexity results are also based on showing that for a single clustering
instance, we can find a partitioning of the algorithm parameter space into
regions where the output clustering is constant. The families of linkage
procedures they study have a single parameter, while our linkage algorithm and
metric families have multiple. Moreover, they suppose we are given a fixed
metric for each clustering instance and do not study the problem of learning an
application-specific metric. \citet{Balcan18:kmeans} study the related problem
of learning the best initialization procedure and local search method to use in
a clustering algorithm inspired by Lloyd's method for $k$-means clustering.
Their sample complexity results are again based on demonstrating that for any
clustering instance, there exists a partitioning of the parameter space on which
the algorithm's output is constant. The parameter space partitions in both of
these related works are defined by linear separators. Due to the interactions
between the distance metric and the linkage algorithm, our partitions are
defined by quadratic functions.

The procedures proposed by prior work for finding an empirically optimal
algorithm for a collection of problem instances roughly fall into two
categories: combinatorial approaches and approaches based on an execution-tree
data structure. \citet{Gupta17:PAC} and \citet{Balcan17:Learning} are two
examples of the combinatorial approach. They show that the boundaries in the
constant-output partition of the algorithm parameter space always occur at the
solutions to finitely many equations that depend on the problem instance. To
find an empirically optimal algorithm, they find all solutions to these
problem-dependent equations to explicitly construct a partition of the parameter
space. Unfortunately, only a small subset of the solutions are actual boundaries
in the partition. Consequently, their partitions contain many extra regions and
suffer from long running times. The execution-tree based approaches find the
coarsest possible partitioning of the parameter space such that the algorithm
output is constant. \citet{Balcan18:kmeans} and \citet{Balcan18:MIP} both use
execution trees to find empirically optimal algorithm parameters for different
algorithm families. However, the specific algorithms used to construct and
enumerate the execution tree are different from those explored in this paper and
are not suitable in our setting.

%% file: joint-family.tex
The problem we study is as follows. Let $\cX$ be a data domain. A clustering
instance consists of a point set $S = \{x_1, \dots, x_n\} \subset \cX$ and an
(unknown) target clustering $\target = (C_1, \dots, C_k)$, where the sets $C_1,
\dots, C_k$ partition $S$ into $k$ clusters. Linkage-based clustering algorithms
output a hierarchical clustering of the input data, represented by a cluster
tree. We measure the agreement of a cluster tree $T$ with the target clustering
$\target = (C_1, \dots, C_k)$ in terms of the Hamming distance between $\target$
and the closest pruning of $T$ into $k$ clusters (i.e., $k$ disjoint subtrees
that contain all the leaves of $T$). More formally, we define the loss
$
  \ell(T, \target)
  = \min_{P_1, \dots, P_k} \min_{\sigma \in \mathbb{S}_n}
  \frac{1}{|S|} \sum_{i=1}^k |C_i \setminus P_{\sigma_i}|,
$
where $A \setminus B$ denotes set difference, the first minimum is over all
prunings $P_1, \dots, P_k$ of the cluster tree $T$, and the second minimum is
over all permutations of the $k$ cluster indices. This formulation allows us to
handle the case where each clustering task has a different number of clusters,
and where the desired number might not be known in advance. Our analysis applies
to any loss function $\ell$ measuring the quality of the output cluster tree
$T$, but we focus on the Hamming distance for simplicity. Given a distribution
$\mathcal{D}$ over clustering instances (i.e., point sets together with target
clusterings), our goal is to find the algorithm $A$ from a family $\cA$ with the
lowest expected loss for an instance sampled from $\mathcal{D}$. As training
data, we assume that we are given an i.i.d. sample of clustering instances
annotated with their target clusterings drawn from the application distribution
$\mathcal{D}$.

We study linkage-based clustering algorithms. These algorithms construct a
hierarchical clustering of a point set by starting with each point belonging to
a cluster of its own and then they repeatedly merge the closest pair of clusters
until only one remains. There are two distinct notions of distance at play in
linkage-based algorithms: first, the notion of distance between pairs of points
(e.g., Euclidean distance between feature vectors, edit distance between
strings, or the Jaccard distance between sets). Second, these algorithms must
define a distance function between clusters, which we refer to as a merge
function to avoid confusion. A merge function $D$ defines the distance between a
pair of clusters $A, B \subset \cX$ in terms of the pairwise distances given by
a metric $d$ between their points. For example, single linkage uses the merge
function $\Dmin(A, B; d) = \min_{a \in A, b \in B} d(a, b)$ and complete linkage
uses the merge function $\Dmax(A, B; d) = \max_{a \in A, b \in B} d(a,b)$.

Our parameterized family of linkage-based clustering algorithms allows us to
vary both the metric used to measure distances between points, as well as the
merge function used to measure distances between clusters. To vary the metric,
we suppose we have access to $\numMetrics$ metrics $d_1, \dots, d_\numMetrics$
defined on our data universe $\cX$, and our goal is to find the best convex
combination of those metrics. That is, for any parameter vector $\vec{\beta} \in
\Delta_\numMetrics = \{ \vec{\beta} \in [0,1]^\numMetrics \mid \sum_i \beta_i =
1\}$, we define a metric $d_{\vec{\beta}}(x,x') = \sum_i \beta_i \cdot
d_i(x,x')$. This definition is suitable across a wide range of applications,
since it allows us to learn the best combination of a given set of metrics for
the application at hand. Similarly, for varying the merge function, we suppose
we have $\numMerges$ merge functions $D_1, \dots, D_\numMerges$. For any
parameter $\vec{\alpha} \in \Delta_\numMerges$, define the merge function
$D_{\vec{\alpha}}(A,B; d) = \sum_i \alpha_i D_i(A, B; d)$. For each pair of
parameters $\vec{\beta} \in \Delta_\numMetrics$ and $\vec{\alpha} \in
\Delta_\numMerges$, we obtain a different clustering algorithm (i.e., one that
repeatedly merges the pair of clusters minimizing $D_{\vec{\alpha}}(\cdot,
\cdot; d_{\vec{\beta}})$). Pseudocode for this method is given in
\Cref{alg:jointLinkage}. In the pseudocode, clusters are represented by binary
trees with leaves correspond to the points belonging to that cluster.
For any clustering instance $S \subset \cX$, we let $\cA_{\vec{\alpha},
\vec{\beta}}(S)$ denote the cluster tree output by \Cref{alg:jointLinkage} when
run with parameter vectors $\vec{\alpha}$ and $\vec{\beta}$.

\begin{algorithm}
\textbf{Input:} Metrics $d_1, \dots, d_\numMetrics$, merge functions $D_1, \dots, D_\numMerges$, points $x_1, \dots, x_n \in \cX$, parameters $\vec{\alpha}$ and $\vec{\beta}$.
\begin{enumerate}[nosep, leftmargin=*]
\item Let $\mathcal{N} = \{\leaf(x_1), \dots, \leaf(x_n)\}$ be the initial set of nodes (one leaf per point).
\item While $|\mathcal{N}| > 1$
\begin{enumerate}[nosep, leftmargin=*]
  \item Let $A, B \in \mathcal{N}$ be the clusters in $\mathcal{N}$ minimizing $D_{\vec{\alpha}}(A, B; d_{\vec{\beta}})$.
  \item Remove clusters $A$ and $B$ from $\mathcal{N}$ and add $\node(A,B)$ to $\mathcal{N}$.
\end{enumerate}
\item Return the cluster tree (the only element of $\mathcal{N}$).
\end{enumerate}
\caption{Linkage Clustering}
\label{alg:jointLinkage}
\end{algorithm}

First, we provide sample complexity results that hold for any collection of
metrics $d_1, \dots, d_\numMetrics$ and any collection of merge functions $D_1,
\dots, D_\numMerges$ that belong to the following family:

\begin{defn}
  A merge function $D$ is \emph{2-point-based} if for any pair of clusters $A,B
  \subset \cX$ and any metric $d$, there exists a pair of points $(a, b) \in A
  \times B$ such that $D(A,B;d) = d(a,b)$. Moreover, the pair of points defining
  the merge distance must depend only on the ordering of pairwise distances.
  More formally, if $d$ and $d'$ are two metrics s.t. for all $a, a' \in A$ and
  $b, b' \in B$, we have $d(a,b) \leq d(a',b')$ if and only if $d'(a,b) \leq
  d'(a',b')$, then $D(A,B;d) = d(a,b)$ implies that $D(A,B;d') = d'(a,b)$.
\end{defn}

For example, both single and complete linkage are 2-point-based merge functions,
since they output the distance between the closest or farthest pair of points,
respectively.
%

\begin{restatable}{thm}{thmSampleComplexity} \label{thm:sampleComplexity}
  Fix any metrics $d_1, \dots, d_\numMetrics$, 2-point-based merge functions
  $D_1, \dots, D_\numMerges$, and distribution $\mathcal{D}$ over clustering
  instances with at most $n$ points. For any parameters $\epsilon > 0$ and
  $\delta > 0$, let $(S_1, \target_1), \dots, (S_N, \target_N)$ be an i.i.d.
  sample of $N = O\left(\frac{1}{\epsilon^2}\left((\numMerges +
  \numMetrics)^2(\log(\numMerges + \numMetrics) + \numMerges \log n) + \log
  \frac{1}{\delta} \right)\right) = \tilde O\left(\frac{(\numMerges +
  \numMetrics)^2 \numMerges}{\epsilon^2}\right)$ clustering instances with
  target clusterings drawn from $\cD$. Then with probability at least $1-\delta$
  over the draw of the sample, we have
  \[
  \sup_{(\vec{\alpha}, \vec{\beta}) \in \Delta_\numMerges \times \Delta_\numMetrics}
  \left|
  \frac{1}{N} \sum_{i=1}^N \ell(\cA_{\vec{\alpha},\vec{\beta}}(S_i), \target_i)
  -
  \expect_{(S,\target) \sim \cD}\bigl[
  \ell(\cA_{\vec{\alpha},\vec{\beta}}(S), \target)
  \bigr]
  \right|
  \leq \epsilon.
  \]
\end{restatable}

The key step in the proof of \Cref{thm:sampleComplexity} is to show that for any
clustering instance $S$ with target clustering $\target$, the function
$(\vec{\alpha}, \vec{\beta}) \mapsto \ell(\cA_{\vec{\alpha}, \vec{\beta}}(S),
\target)$ is piecewise constant with not too many pieces and where each piece is
simple. Intuitively, this guarantees that for any collection of clustering
instances, we cannot see too much variation in the loss of the algorithm on
those instances as we vary over the parameter space. \citet{Balcan19:General}
give precise sample complexity guarantees for algorithm configuration problems
when the cost is a piecewise-structured function of the algorithm parameters. In
\Cref{app:jointFamily} we apply their general bounds together with our
problem-specific structural to prove \Cref{thm:sampleComplexity}. In the
remainder of this section, we prove the key structural property.

We let $\vec{\zeta} = (\vec{\alpha}, \vec{\beta}) \in \Delta_\numMerges \times
\Delta_\numMetrics$ denote a pair of parameter vectors for
\Cref{alg:jointLinkage}, viewed as a vector in $\reals^{\numMerges +
\numMetrics}$. Our parameter space partition will be induced by the sign-pattern
of $M$ quadratic functions.

\begin{defn}[Sign-pattern Partition]
  The sign-pattern partition of $\reals^p$ induced $M$ functions $f_1, \dots,
  f_M : \reals^p \to \reals$ is defined as follows: two points $\vec{\zeta}$ and
  $\vec{\zeta'}$ belong to the same region in the partition iff
  $\sign(f_i(\vec{\zeta})) = \sign(f_i(\vec{\zeta}'))$ for all $i \in [M]$. Each
  region is of the form $\region = \{ \vec{\zeta} \in \reals^p | F(\vec{\zeta})
  = \vec{b}\}$, for some sign-pattern vector $\vec{b} \in \{\pm 1\}^M$.
  %
\end{defn}

%

We show that for any fixed metrics $d_1, \dots, d_\numMetrics$ and clustering
instance $S = \{x_1, \dots, x_n\} \subset \cX$, we can find a sign-pattern
partitioning of $\Delta_\numMetrics$ induced by linear functions such that, on
each region, the ordering over pairs of points in $S$ induced by the metric
$d_{\vec{\beta}}$ is constant. An important consequence of this result is that
for each region $\region$ in this partitioning of $\Delta_\numMetrics$, the
following holds: For any 2-point-based merge function $D$ and any pair of
clusters $A,B \subset S$, there exists a pair of points $(a,b) \in A \times B$
such that $D(A, B; d_{\vec{\beta}}) = d_{\vec{\beta}}(a, b)$ for all
$\vec{\beta} \in \region$. In other words, restricted to $\vec{\beta}$
parameters belonging to $\region$, the same pair of points $(a,b)$ defines the
$D$-merge distance for the clusters $A$ and $B$.

\begin{restatable}{lem}{lemDistanceOrdering} \label{lem:distanceOrdering}
  Fix any metrics $d_1, \dots, d_\numMetrics$ and a clustering instance $S
  \subset \cX$. There exists a set $\cH$ of $O(|S|^4)$ linear functions mapping
  $\reals^\numMetrics$ to $\reals$ with the following property: if two metric
  parameters $\vec{\beta}, \vec{\beta}' \in \Delta_\numMetrics$ belong to the
  same region in the sign-pattern partition induced by $\cH$, then the ordering
  over pairs of points in $S$ given by $d_{\vec{\beta}}$ and $d_{\vec{\beta}'}$
  are the same. That is, for all points $a, b, a', b' \in S$ we have
  $d_{\vec{\beta}}(a,b) \leq d_{\vec{\beta}}(a', b')$ iff $d_{\vec{\beta}'}(a,b)
  \leq d_{\vec{\beta}'}(a',b')$.
\end{restatable}
\begin{proof}[Proof sketch]
  For any pair of points $a, b \in S$, the distance $d_{\vec{\beta}}(a,b)$ is a
  linear function of the parameter $\vec{\beta}$. Therefore, for any four points
  $a, b, a', b' \in S$, we have that $d_{\vec{\beta}}(a,b) \leq
  d_{\vec{\beta}}(a', b')$ iff $h_{a,b,a',b'}(\vec{\beta}) \leq 0$, where
  $h_{a,b,a',b'}$ is the linear function given by $h_{a,b,a',b'}(\vec{\beta}) =
  d_{\vec{\beta}}(a,b) - d_{\vec{\beta}}(a',b')$. Let $\cH = \{h_{a,b,a',b'}
  \mid a,b,a',b' \in S\}$ be the collection of all such linear functions arising
  from any subset of 4 points. On each region of the sign-pattern partition
  induced by $\cH$, all comparisons of pairwise distances in $S$ are fixed,
  implying that the ordering over pairs of points in $S$ is fixed.
\end{proof}

Building on \Cref{lem:distanceOrdering}, we now prove the main structural
property of \Cref{alg:jointLinkage}. We argue that for any clustering instance
$S \subset \cX$, there is a partition induced by quadratic functions of
$\Delta_\numMerges \times \Delta_\numMetrics \subset \reals^{\numMerges +
\numMetrics}$ over $\vec{\alpha}$ and $\vec{\beta}$ into regions such that on
each region, the ordering over all pairs of clusters according to the merge
distance $D_{\vec{\alpha}}(\cdot, \cdot; d_{\vec{\beta}})$ is fixed. This
implies that for all $(\vec{\alpha}, \vec{\beta})$ in one region of the
partition, the output of \Cref{alg:jointLinkage} when run on $S$ is constant,
since the algorithm output only depends on the ordering over pairs of clusters
in $S$ given by $D_{\vec{\alpha}}(\cdot, \cdot; d_{\vec{\beta}})$.

\begin{restatable}{lem}{lemQuadraticPartition} \label{lem:quadraticPartition}
  Fix any metrics $d_1, \dots, d_\numMetrics$, any 2-point-based merge functions
  $D_1, \dots, D_\numMerges$, and clustering instance $S \subset \cX$. There
  exists a set $\cQ$ of $O(|S|^{4\numMerges})$ quadratic functions defined on
  $\reals^{\numMerges + \numMetrics}$ so that if parameters
  $(\vec{\alpha},\vec{\beta})$ and $(\vec{\alpha}',\vec{\beta}')$ belong to the
  same region of the sign-pattern partition induced by $\cQ$, then the ordering
  over pairs of clusters in $S$ given by $D_{\vec{\alpha}}(\cdot, \cdot;
  d_{\vec{\beta}})$ and $D_{\vec{\alpha}'}(\cdot, \cdot; d_{\vec{\beta}'})$ is
  the same. That is,  for all clusters $A, B, A', B' \subset S$, we have that
  $D_{\vec{\alpha}}(A, B; d_{\vec{\beta}}) \leq D_{\vec{\alpha}}(A', B';
  d_{\vec{\beta}})$ iff $D_{\vec{\alpha}'}(A, B; d_{\vec{\beta}'}) \leq
  D_{\vec{\alpha}'}(A', B'; d_{\vec{\beta}'})$.
\end{restatable}
\begin{proof}[Proof sketch]
  Let $\cH$ be the linear functions constructed in \Cref{lem:distanceOrdering}
  and fix any region $\region$ in the sign-pattern partition induced by $\cH$.
  For any $i \in [\numMerges]$, since the merge function $D_i$ is 2-point-based
  and the ordering over pairs of points according to $d_{\vec{\beta}}(\cdot,
  \cdot)$ in the region $\region$ is fixed, for any clusters $A, B, A', B'
  \subset S$ we can find points $a_i, b_i, a_i', b_i'$ such that $D_i(A, B;
  d_{\vec{\beta}}) = d_{\vec{\beta}}(a_i, b_i)$ and $D_i(A', B';
  d_{\vec{\beta}}) = d_{\vec{\beta}}(a'_i, b'_i)$ for all $\vec{\beta} \in R$.
  Therefore, expanding the definition of $D_{\vec{\alpha}}$, we have that
  $D_{\vec{\alpha}}(A, B; d_{\vec{\beta}}) \leq D_{\vec{\alpha}}(A', B';
  d_{\vec{\beta}})$ iff $q_{A,B,A',B'}(\vec{\alpha}, \vec{\beta}) \leq 0$, where
  $q_{A,B,A',B'}(\vec{\alpha}, \vec{\beta}) = \sum_i \sum_j \alpha_i \beta_j
  (d_j(a_i, b_i) - d_j(a_i', b_i'))$. Observe that the coefficients of each
  quadratic function depend only on $4\numMerges$ points in $S$, so there are
  only $O(|S|^{4 \numMerges})$ possible quadratics of this from collected across
  all regions $\region$ in the sign-pattern partition induced by $\cH$ and all
  subsets of 4 clusters. Together with $\cH$, this set of quadratic functions
  partitions the joint parameter space into regions where the ordering over all
  pairs of clusters is fixed.
\end{proof}

A consequence of \Cref{lem:quadraticPartition} is that for any clustering
instance $S$ with target clustering $\target$, the function $(\vec{\alpha},
\vec{\beta}) \mapsto \ell(\cA_{\vec{\alpha}, \vec{\beta}}(S), \target)$ is
piecewise constant, where the constant partitioning is the sign-pattern
partition induced by $O(|S|^{4\numMerges})$ quadratic functions. Combined with
the general theorem of \citet{Balcan19:General}, this proves
\Cref{thm:sampleComplexity}.

\paragraph{Extensions.} The above analysis can be extended to handle several
more general settings. First, we can accommodate many specific merge functions
that are not included in the 2-point-based family, at the cost of increasing the
number of quadratic functions $|\cQ|$ needed in \Cref{lem:quadraticPartition}.
For example, if one of the merge functions is average linkage, $\Davg(A,B; d) =
\frac{1}{|A|\cdot |B|} \sum_{a \in A, b \in B} d(a,b)$, then $|\cQ|$ will be
exponential in the dataset size $n$. Fortunately, our sample complexity analysis
depends only on $\log(|\cQ|)$, so this still leads to non-trivial sample
complexity guarantees (though the computational algorithm selection problem
becomes harder). We can also extend the analysis to more intricate methods for
combining the metrics and merge functions. For example, our analysis applies to
polynomial combinations of metrics and merges at the cost of increasing the
complexity of the functions defining the piecewise constant partition.

%% file: erm-algorithms.tex
In this section we provide efficient algorithms for learning low-loss clustering
algorithms and metrics for application-specific distributions $\cD$ defined over
clustering instances. We begin by focusing on the special case where we have a
single metric and our goal is to learn the best combination of two merge
functions (i.e., $\numMetrics = 1$ and $\numMerges = 2$). This special case is
already algorithmically interesting. Next, we show how to apply similar
techniques to the case of learning the best combination of two metrics when
using the complete linkage merge function (i.e., $\numMetrics = 2$ and
$\numMerges = 1$). Finally, we discuss how to generalize our techniques to other
cases.

\paragraph{Learning the Merge Function.} We will use the following simplified
notation for mixing two base merge functions $D_0(A,B;d)$ and $D_1(A,B;d)$: for
each parameter $\alpha \in [0,1]$, let $D_\alpha(A, B; d) = (1-\alpha)D_0(A, B;
d) + \alpha D_1(A, B; d)$ denote the convex combination with weight $(1-\alpha)$
on $D_0$ and weight $\alpha$ on $D_1$. We let $\mergeAlg{\alpha}(S; D_0, D_1)$
denote the cluster tree produced by the algorithm with parameter $\alpha$, and
$\mergeFamily(D_0, D_1) = \{\mergeAlg{\alpha}(\cdot; D_0, D_1) \mid \alpha \in
[0,1]\}$ denote the parameterized algorithm family.

Our goal is to design efficient procedures for finding the algorithm from
$\mergeFamily(D_0, D_1)$ (and more general families) that has the lowest average
loss on a sample of labeled clustering instances $(S_1, \target_1), \dots, (S_N,
\target_N)$ where $\target_i = (C_1^{(i)}, \dots, C^{(i)}_{k_i})$ is the target
clustering for instance $S_i$. Recall that the loss function $\ell(T, \target)$
computes the Hamming distance between the target clustering $\target$ and the
closest pruning of the cluster tree $T$. Formally, our goal is to solve the
following optimization problem:
$
  \argmin_{\alpha \in [0,1]} \frac{1}{N}\sum_{i=1}^N \ell(\mergeAlg{\alpha}(S_i; D_0, D_1), \target_i).
$

The key challenge is that, for a fixed clustering instance $S$ we can partition
the parameter space $[0,1]$ into finitely many intervals such that for each
interval $I$, the cluster tree output by the algorithm $\mergeAlg{\alpha}(S;
D_0, D_1)$ is the same for every parameter in $\alpha \in I$. It follows that
the loss function is a piecewise constant function of the algorithm parameter.
Therefore, the optimization problem is non-convex and the loss derivative is
zero wherever it is defined, rendering gradient descent and similar algorithms
ineffective.

We solve the optimization problem by explicitly computing the piecewise constant
loss function for each instance $S_i$. That is, for instance $i$ we find a
collection of discontinuity locations
$0~=~c^{(i)}_0~<~\dots~<~c^{(i)}_{M_i}~=~1$ and values $v^{(i)}_1, \dots,
v^{(i)}_{M_i} \in \reals$ so that for each $j \in [M_i]$, running the algorithm
on instance $S_i$ with a parameter in $[c^{(i)}_{j-1}, c^{(i)}_j)$ has loss
equal to $v^{(i)}_j$. Given this representation of the loss function for each of
the $N$ instances, finding the parameter with minimal average loss can be done
in $O(M \log(M))$ time, where $M = \sum_i M_i$ is the total number of
discontinuities from all $N$ loss functions.
The bulk of the computational cost is incurred by computing the piecewise
constant loss functions, which we focus on for the rest of the section.

We exploit a more powerful structural property of the algorithm family to
compute the piecewise constant losses: for a clustering instance $S$ and any
length $t$, the sequence of first $t$ merges performed by the algorithm is a
piecewise constant function of the parameter (our sample complexity results only
used that the final tree is piecewise constant). For length $t = 0$, the
partition is a single region containing all parameters in $[0,1]$, since every
algorithm trivially starts with the empty sequence of merges. For each length $t
> 0$, the piecewise constant partition for the first $t$ merges is a refinement
of the partition for $t-1$ merges. We can represent this sequence of partitions
using a partition tree, where each node in the tree is labeled by an interval,
the nodes at depth $t$ describe the partition of $[0,1]$ after $t$ merges, and
edges represent subset relationships. This tree represents all possible
execution paths for the algorithm family when run on the instance $S$ as we vary
the algorithm parameter. In particular, each path from the root node to a leaf
corresponds to one possible sequence of merges. We therefore call this tree the
\emph{execution tree} of the algorithm family when run on $S$.
\Cref{fig:executionTree} shows an example execution tree for the family
$\mergeFamily(\Dmin, \Dmax)$. To find the piecewise constant loss function for a
clustering instance $S$, it is sufficient to enumerate the leaves of the
execution tree and compute the corresponding losses.
The following result, proved in \Cref{app:ermAlgorithms}, shows that the
execution tree for $\mergeFamily(D_0, D_1)$ is well defined.

\begin{restatable}{lem}{lemMergeExecutionTree}\label{lem:mergeExecutionTree}
  For any merge functions $\DZero$ and $\DOne$ and any clustering instance $S$,
  the execution tree for $\mergeFamily(\DZero, \DOne)$ when run on $S$ is well
  defined. That is, there exists a partition tree s.t. for any node $v$ at depth
  $t$, the same sequence of first $t$ merges is performed by $\mergeAlg{\alpha}$
  for all $\alpha$ in node $v$'s interval.
\end{restatable}

\begin{figure}
  \centering \includegraphics[width=0.45\textwidth]{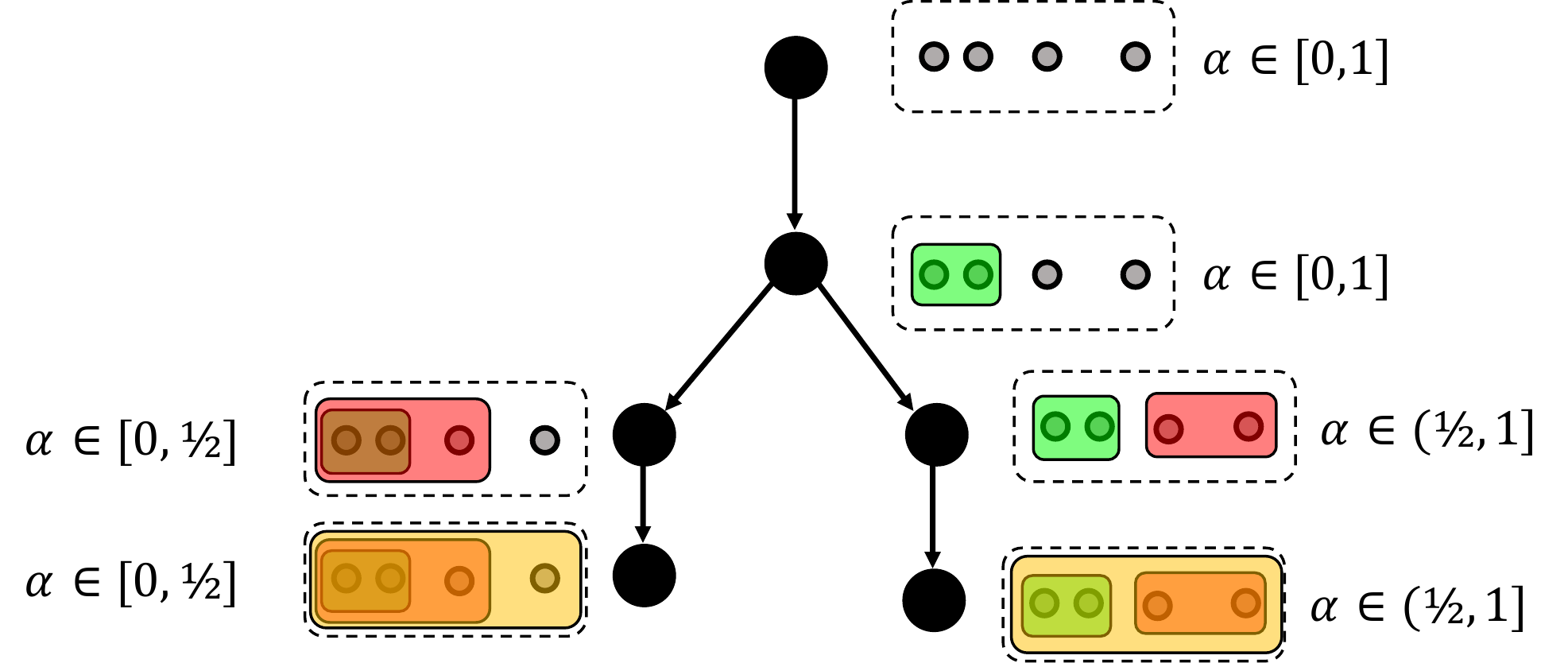}
  \caption{An example of the execution tree of $\mergeFamily(\Dmin, \Dmax)$ for
  a clustering instance with $4$ points. The nested rectangles show the
  clustering at each node.} \label{fig:executionTree}
  \vspace{-0.45cm}
\end{figure}

The fundamental operation required to perform a depth-first traversal of the
execution tree is finding a node's children. That is, given a node, its
parameter interval $[\alo, \ahi)$, and the set of clusters at that node $C_1,
\dots, C_m$, find all possible merges that will be chosen by the algorithm for
$\alpha \in [\alo, \ahi)$. We know that for each pair of merges $(C_i, C_j)$ and
$(C_i', C_j')$, there is a single critical parameter value where the algorithm
switches from preferring to merge $(C_i, C_j)$ to $(C_i', C_j')$. A direct
algorithm that runs in $O(m^4)$ time for finding the children of a node in the
execution tree is to compute all $O(m^4)$ critical parameter values and test
which pair of clusters will be merged on each interval between consecutive
critical parameters. We provide a more efficient algorithm that runs in time
$O(m^2 M)$, where $M \leq m^2$ is the number of children of the node.

Fix any node in the execution tree. Given the node's parameter interval $I =
[\alo, \ahi)$ and the set of clusters $C_1, \dots, C_m$ resulting from that
node's merge sequence, we use a sweep-line algorithm to determine all possible
next merges and the corresponding parameter intervals. First, we calculate the
merge for $\alpha = \alo$ by enumeration in $O(m^2)$ time. Suppose clusters
$C_i$ and $C_j$ are the optimal merge for $\alpha$. We then determine the
largest value $\alpha'$ for which $C_i$ and $C_j$ are still merged by solving
the linear equation $D_{\alpha}(C_i, C_j) = D_{\alpha}(C_k, C_l)$ for all other
pairs of clusters $C_k$ and $C_l$, keeping track of the minimal solution larger
than $\alpha$. Since there are only $O(m^2)$ alternative pairs of clusters, this
takes $O(m^2)$ time. Denote the minimal solution larger than $\alpha$ by $c \in
I$. We are guaranteed that $\mergeAlg{\alpha'}$ will merge clusters $C_i$ and
$C_j$ for all $\alpha' \in [\alpha, c)$. We repeat this procedure starting from
$\alpha = c$ to determine the next merge and corresponding interval, and so on,
sweeping through the $\alpha$ parameter space until $\alpha \geq \ahi$.
\Cref{alg:findMerges} in \Cref{app:ermAlgorithms} provides pseudocode for this
approach and \Cref{fig:sweepline} shows an example. Our next result bounds the
running time of this procedure.

\begin{figure}
  \centering
  \begin{subfigure}[b]{0.35\textwidth}
    \includegraphics[width=0.9\textwidth]{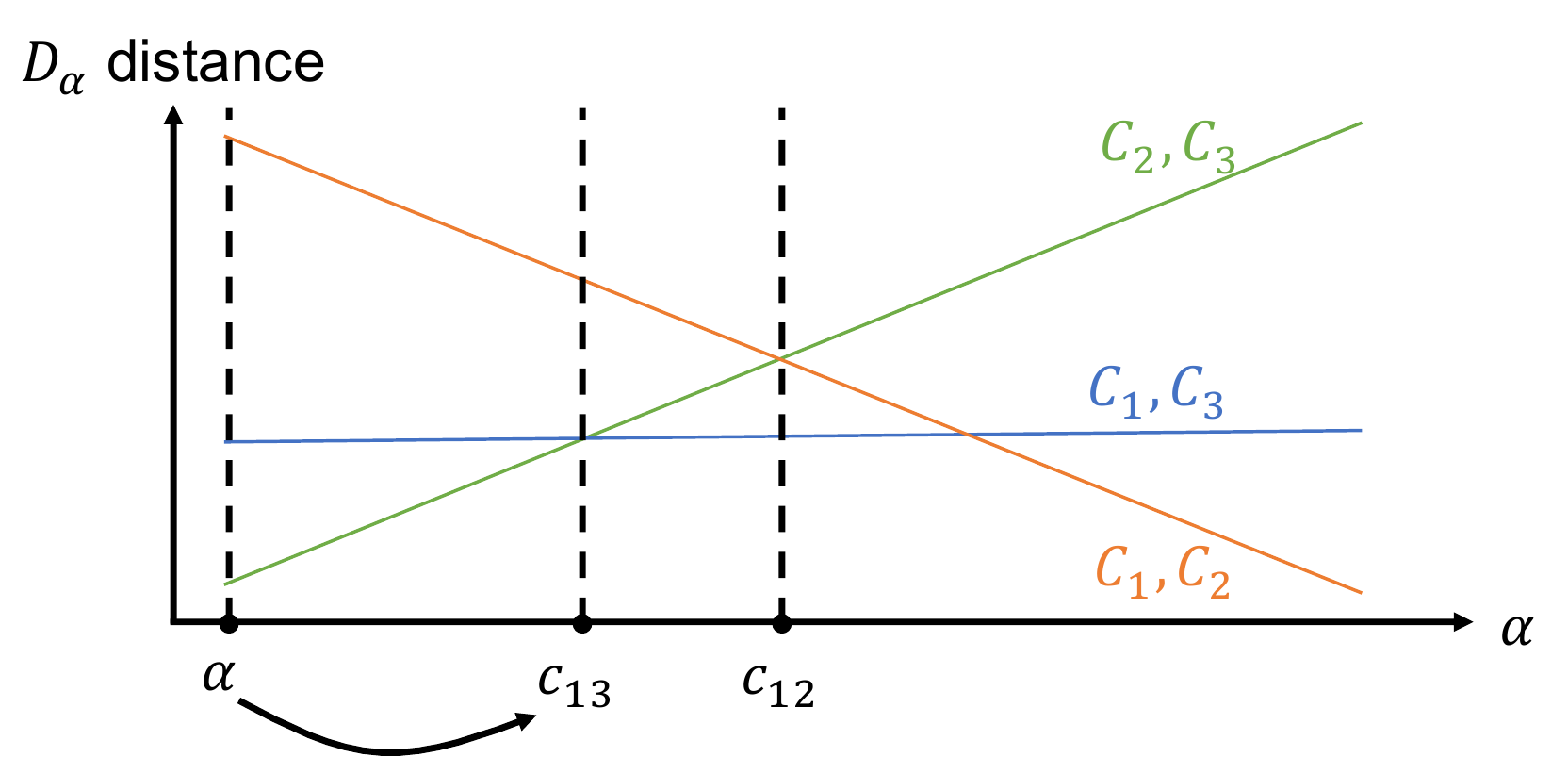}
    \caption{First iteration} \label{fig:sweeplineFirstIter}
  \end{subfigure}
  ~
  \begin{subfigure}[b]{0.35\textwidth}
    \includegraphics[width=0.9\textwidth]{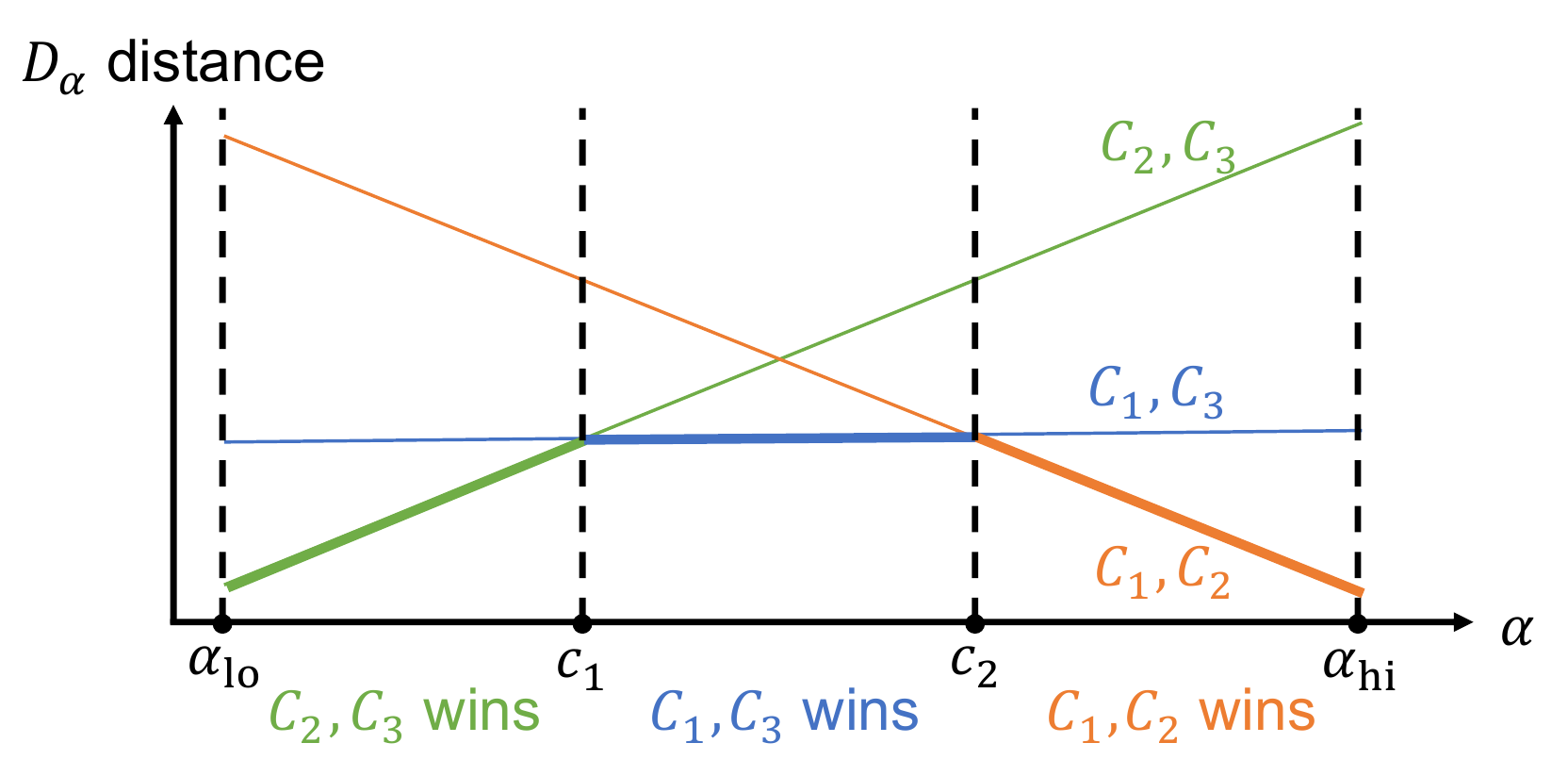}
    \caption{Final output} \label{fig:sweeplineFinalOutput}
  \end{subfigure}
  \caption{Depiction of \Cref{alg:findMerges} when given three clusters, $C_1$,
  $C_2$, and $C_3$. Each line shows the $D_\alpha$-distance between one pair of
  clusters as a function of the parameter $\alpha$. On the first iteration,
  \Cref{alg:findMerges} determines that clusters $C_2$ and $C_3$ are the closest
  for the parameter $\alpha = \alo$, calculates the critical parameter values
  $c_{12}$ and $c_{13}$, and advances $\alpha$ to $c_{13}$. Repeating the
  process partitions of $[\alo, \ahi)$ into merge-constant regions.}
  \label{fig:sweepline}
  \vspace{-0.4cm}
\end{figure}

\begin{restatable}{lem}{lemFindMergesRuntime} \label{lem:findMergesRuntime}
  Let $C_1, \dots, C_m$ be a collection of clusters, $\DZero$ and $\DOne$ be any
  pair of merge functions, and $[\alo, \ahi)$ be a subset of the parameter
  space. If there are $M$ distinct cluster pairs $C_i, C_j$ that minimize
  $\Dalpha(C_i, C_j)$ for values of $\alpha \in [\alo, \ahi)$, then the running
  time of \Cref{alg:findMerges} is $O(M m^2 K)$, where $K$ is the cost of
  evaluating the merge functions $\DZero$ and $\DOne$.
\end{restatable}

With this, our algorithm for computing the piecewise constant loss function for
an instance $S$ performs a depth-first traversal of the leaves of the execution
tree for $\mergeFamily(\DZero, \DOne)$, using \Cref{alg:findMerges} to determine
the children of each node. When we reach a leaf in the depth-first traversal, we
have both the corresponding parameter interval $I \subset [0,1]$, as well as the
cluster tree $T$ such that $\mergeAlg{\alpha}(S) = T$ for all $\alpha \in I$. We
then evaluate the loss $\ell(T, \target)$ to get one piece of the piecewise
constant loss function. Detailed pseudocode for this approach is given in
\Cref{alg:depthFirst} in \Cref{app:ermAlgorithms}.

\begin{restatable}{thm}{thmDepthFirstRuntime} \label{lem:depthFirstRuntime}
  Let $S = \{x_1, \dots, x_n\}$ be a clustering instance and $\DZero$ and
  $\DOne$ be any two merge functions. Suppose that the execution tree of
  $\mergeFamily(\DZero ,\DOne)$ on $S$ has $E$ edges. Then the total running
  time of \Cref{alg:depthFirst} is $O(E n^2 K)$, where $K$ is the cost of
  evaluating $\DZero$ and $\DOne$ once.
\end{restatable}



We can express the running time of \Cref{alg:depthFirst} in terms of the number
of discontinuities of the function $\alpha \mapsto \mergeAlg{\alpha}(S)$. There
is one leaf of the execution tree for each constant interval of this function,
and the path from the root of the execution tree to that leaf is of length
$n-1$. Therefore, the cost associated with that path is at most $O(K n^3)$ and
enumerating the execution tree to obtain the piecewise constant loss function
for a given instance $S$ spends $O(K n^3)$ time for each constant interval of
$\alpha \mapsto \mergeAlg{\alpha}(S)$. In contrast, the combinatorial approach
of \citet{Balcan17:Learning} requires that we run $\alpha$-linkage once for
every interval in their partition of $[0,1]$, which always contains $O(n^8)$
intervals (i.e., it is a refinement of the piecewise constant partition). Since
each run of $\alpha$-Linkage costs $O(K n^2 \log n)$ time, this leads to a
running time of $O(K n^{10} \log n)$. The key advantage of our approach stems
from the fact that the number of discontinuities of the function $\alpha \mapsto
\mergeAlg{\alpha}(S)$ is often several orders of magnitude smaller than
$O(n^8)$.

\paragraph{Learning the Metric.} Next we present efficient algorithms for
computing the piecewise constant loss function for a single clustering instance
when interpolating between two base metrics and using complete linkage. For a
pair of fixed base metrics $d_0$ and $d_1$ and any parameter value $\beta \in
[0,1]$, define $d_\beta(a,b) = (1-\beta)d_0(a,b) + \beta d_1(a,b)$. Let
$\metricAlg{\beta}(S;d_0, d_1)$ denote the output of running complete linkage
with the metric $d_\beta$, and $\metricFamily(d_0, d_1)$ denote the family of
all such algorithms. We prove that for this algorithm family, the execution tree
is well defined and provide an efficient algorithm for finding the children of
each node in the execution tree, allowing us to use a depth-first traversal to
find the piecewise constant loss function for any clustering instance $S$.

\begin{restatable}{lem}{lemMetricExecutionTree}\label{lem:metricExecutionTree}
  For any metrics $\dZero$ and $\dOne$ and any clustering instance $S$, the
  execution tree for the family $\metricFamily(\dZero, \dOne)$ when run on $S$
  is well defined. That is, there exists a partition tree s.t. for any node $v$
  at depth $t$, the same sequence of first $t$ merges is performed by
  $\metricAlg{\beta}$ for all $\beta$ in node $v$'s interval.
\end{restatable}

Next, we provide an efficient procedure for determining the children of a node
$v$ in the execution tree of $\metricFamily(\dZero, \dOne)$. Given the node's
parameter interval $I = [\blo, \bhi)$ and the set of clusters $C_1, \dots, C_m$
resulting from that node's sequence of merges, we again use a sweep-line
procedure to find the possible next merges and the corresponding parameter
intervals. First, we determine the pair of clusters that will be merged by
$\metricAlg{\beta}$ for $\beta = \blo$ by enumerating all pairs of clusters.
Suppose the winning pair is $C_i$ and $C_j$ and let $x \in C_i$ and $x' \in C_j$
be the farthest pair of points between the two clusters. Next, we find the
largest value of $\beta'$ for which we will still merge the clusters $C_i$ and
$C_j$. To do this, we enumerate all other pairs of clusters $C_k$ and $C_l$ and
all pairs of points $y \in C_k$ and $y' \in C_l$, and solve the linear equation
$\metricp{\beta'}(x,x') = \metricp{\beta}(y,y')$, keeping track of the minimal
solution larger than $\beta$. Denote the minimal solution larger than $\beta$ by
$c$. We are guaranteed that for all $\beta' \in [\beta, c)$, the pair of
clusters merged will be $C_i$ and $C_j$. Then we repeat the process with $\beta
= c$ to find the next merge and corresponding interval, and so on, until $\beta
\geq \bhi$. Pseudocode for this procedure is given in
\Cref{alg:findMergesMetric} in \Cref{app:ermAlgorithms}. The following Lemma
bounds the running time:

\begin{restatable}{lem}{lemFindMergesRuntimeMetric}\label{lem:findMergesRuntimeMetric}
  Let $C_1, \dots, C_m$ be a collection of clusters, $\dZero$ and $\dOne$ be any
  pair of metrics, and $[\blo, \bhi)$ be a subset of the parameter space. If
  there are $M$ distinct cluster pairs $C_i, C_j$ that complete linkage would
  merge when using the metric $\dbeta$ for $\beta \in [\blo, \bhi)$, the running
  time of \Cref{alg:findMergesMetric} is $O(M n^2)$.
\end{restatable}

Our algorithm for computing the piecewise constant loss function for an instance
$S$ is almost identical for the case of the merge function: it performs a
depth-first traversal of the leaves of the execution tree for
$\metricFamily(\dZero, \dOne)$, using \Cref{alg:findMergesMetric} to determine
the children of each node. Detailed pseudocode for this approach is given in
\Cref{alg:depthFirstMetric} in \Cref{app:ermAlgorithms}. The following Theorem
characterizes the overall running time of the algorithm.

\begin{restatable}{thm}{thmDepthFirstRuntimeMetric} \label{lem:depthFirstRuntimeMetric}
  Let $S = \{x_1, \dots, x_n\}$ be a clustering instance and $\dZero$ and
  $\dOne$ be any two merge functions. Suppose that the execution tree of
  $\metricFamily(\dZero ,\dOne)$ on $S$ has $E$ edges. Then the total running
  time of \Cref{alg:depthFirstMetric} is $O(E n^2)$.
\end{restatable}

\paragraph{General algorithm families.} Our efficient algorithm selection
procedures have running time that scales with the true number of discontinuities
in each loss function, rather than a worst-case upper bound. The two special
cases we study each have one-dimensional parameter spaces, so the partition at
each level of the tree always consists of a set of intervals. This approach can
be extended to the case when we have multiple merge functions and metrics,
except now the partition at each node in the tree will be a sign-pattern
partition induced by quadratic functions.

%% file: experiments-new.tex
In this section we evaluate the performance of our learning procedures when
finding algorithms for application-specific clustering distributions. Our
experiments demonstrate that the best algorithm for different applications
varies greatly, and that in many cases we can have large gains in cluster
quality using a mixture of base merge functions or metrics.

\paragraph{Experimental setup.} In each experiment we define a distribution
$\mathcal{D}$ over clustering tasks. For each clustering instance, the loss of
the cluster tree output by a clustering algorithm is measured in terms of the
loss $\ell(S, \target)$, which computes the Hamming distance between the target
clustering and the closest pruning of the cluster tree. We draw $N$ sample
clustering tasks from the given distribution and use the algorithms developed in
\Cref{sec:ermAlgorithms} to exactly compute the average empirical loss for every
algorithm in one algorithm family. The theoretical results from
\Cref{sec:jointFamily} ensure that these plots generalize to new samples from
the same distribution, so our focus is on demonstrating empirical improvements
in clustering loss obtained by learning the merge function or metric.

\paragraph{Clustering distributions.} Most of our clustering distributions are
generated from classification datasets by sampling a subset of the dataset and
using the class labels as the target clustering. We briefly describe our
instance distributions together with the metrics used for each below. Complete
details for the distributions can be found in \Cref{app:experiments}.

\itparagraph{MNIST Subsets.} The MNIST dataset \citep{LeCun98:MNIST} contains
images of hand-written digits from $0$ to $9$. We generate a random clustering
instance from this data by choosing $k = 5$ random digits and sampling $200$
images from each digit, giving a total of $n = 1000$ images. We measure distance
between any pairs of images using the Euclidean distance between their pixel
intensities.

\itparagraph{CIFAR-10 Subsets.} We similarly generate clustering instances from
the CIFAR-10 dataset~\citep{Krizhevsky09:CIFAR}. To generate an instance, we
select $k = 5$ classes at random and then sample $50$ images from each class,
leading to a total of $n = 250$ images. We measure distances between examples
using the cosine distance between feature embedding extracted from a pre-trained
Google inception network~\citep{Szegedy15:Inception}.

\itparagraph{Omniglot Subsets.} The Omniglot dataset~\citep{Lake15:Omniglot}
contains written characters from 50 alphabets, with a total of 1623 different
characters. To generate a clustering instance from the omniglot data, we choose
one of the alphabets at random, we sample $k$ from $\{5, \dots, 10\}$ uniformly
at random, choose $k$ random characters from the alphabet, and include all 20
examples of those characters in the clustering instance. We use two metrics for
the omniglot data: first, cosine distances between neural network feature
embeddings of the character images from a simplified version of
AlexNet~\citep{Krizhevsky12:Imagenet}. Second, each character is also described
by a ``stroke'', which is a sequence of coordinates $(x_t, y_t)_{t=1}^T$
describing the trajectory of the pen when writing the character. We hand-design
a metric based on the stroke data: the distance between a pair of characters is
the average distance from a point on either stroke to its nearest neighbor on
the other stroke. A formal definition is given in the appendix.

\itparagraph{Places2 Subsets.} The Places2 dataset consists of images of 365
different place categories, including ``volcano'', ``gift shop'', and
``farm''~\citep{Zhou17:Places}. To generate a clustering instance from the
places data, we choose $k$ randomly from $\{5, \dots, 10\}$, choose $k$ random
place categories, and then select $20$ random examples from each chosen
category. We use two metrics for this data distribution. First, we use cosine
distances between feature embeddings generated by a VGG16
network~\citep{Simonyan15:VGG} pre-trained on ImageNet~\citep{Deng09:Imagenet}.
Second, we compute color histograms in HSV space for each image and use the
cosine distance between the histograms.

\itparagraph{Places2 Diverse Subsets.} We also construct an instance
distribution from a subset of the Places2 classes which have diverse color
histograms. We expect the color histogram metric to perform better on this
distribution. To generate a clustering instance, we pick $k = 4$ classes from
aquarium, discotheque, highway, iceberg, kitchen, lawn, stage-indoor, underwater
ocean deep, volcano, and water tower. We include $50$ randomly sampled images
from each chosen class, leading to a total of $n = 200$ points per instance.

\itparagraph{Synthetic Rings and Disks.} We consider a two
dimensional synthetic distribution where each clustering instance has 4
clusters, where two are ring-shaped and two are disk-shaped. To generate each
instance we sample 100 points uniformly at random from each ring or disk. The
two rings have radiuses $0.4$ and $0.8$, respectively, and are both centered at
the origin. The two disks have radius $0.4$ and are centered at $(1.5, 0.4)$ and
$(1.5, -0.4)$, respectively. For this data, we measure distances between points
in terms of the Euclidean distance between them.

\paragraph{Results.} \textit{Learning the Merge Function.}
Figure~\ref{fig:merge} shows the average loss when interpolating between single
and complete linkage as well as between average and complete linkage for each of
the clustering instance distributions described above. For each value of the
parameter $\alpha \in [0,1]$, we report the average loss over $N = 1000$ i.i.d.
instances drawn from the corresponding distribution. We see that the optimal
parameters vary across different clustering instances. For example, when
interpolating between single and complete linkage, the optimal parameters are
$\alpha = 0.874$ for MNIST, $\alpha = 0.98$ for CIFAR-10, $\alpha = 0.179$ for
Rings and Disks, and $\alpha = 0.931$ for Omniglot. Moreover, using the
parameter that is optimal for one distribution on another would lead to
significantly worse clustering performance. Next, we also see that for different
distributions, it is possible to achieve non-trivial improvements over single,
complete, and average linkage by interpolating between them. For example, on the
Rings and Disks distribution we see an improvement of almost $0.2$ error,
meaning that an additional 20\% of the data is correctly clustered.

\begin{figure}[h]
  \centering
  \begin{subfigure}[b]{0.235\textwidth}
    \includegraphics[width=\textwidth]{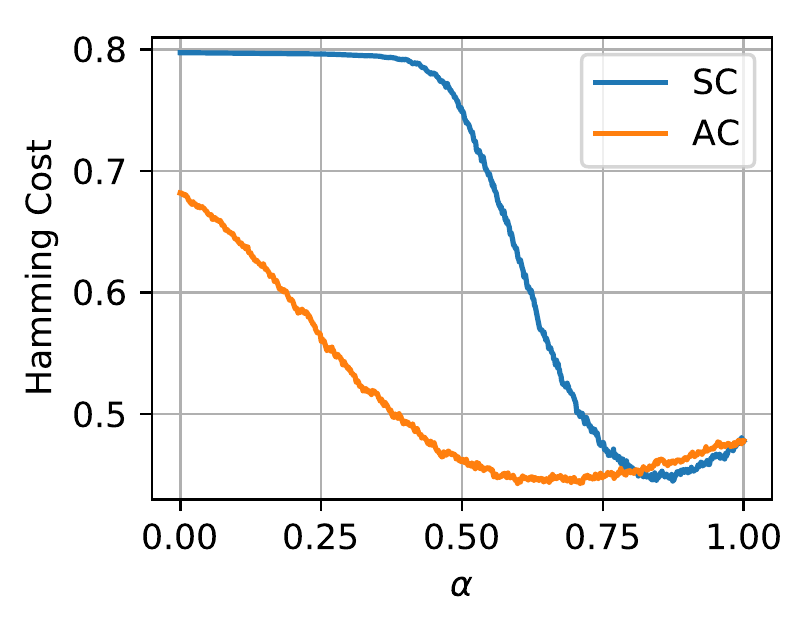}
    \caption{MNIST} \label{fig:mnistMerge}
  \end{subfigure}
  ~
  \begin{subfigure}[b]{0.235\textwidth}
    \includegraphics[width=\textwidth]{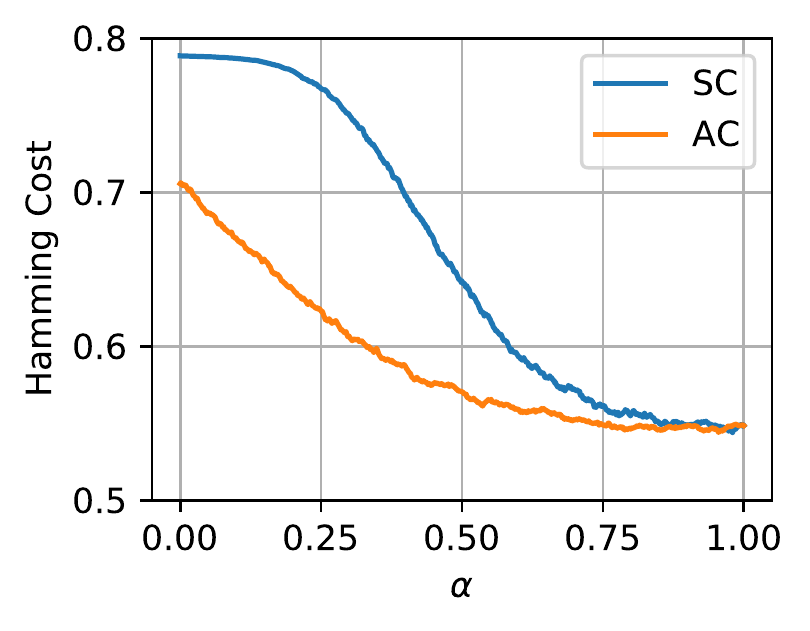}
    \caption{CIFAR-10} \label{fig:cifar10Merge}
  \end{subfigure}
  ~
  \vspace{-0.2em}
  \begin{subfigure}[b]{0.235\textwidth}
    \includegraphics[width=\textwidth]{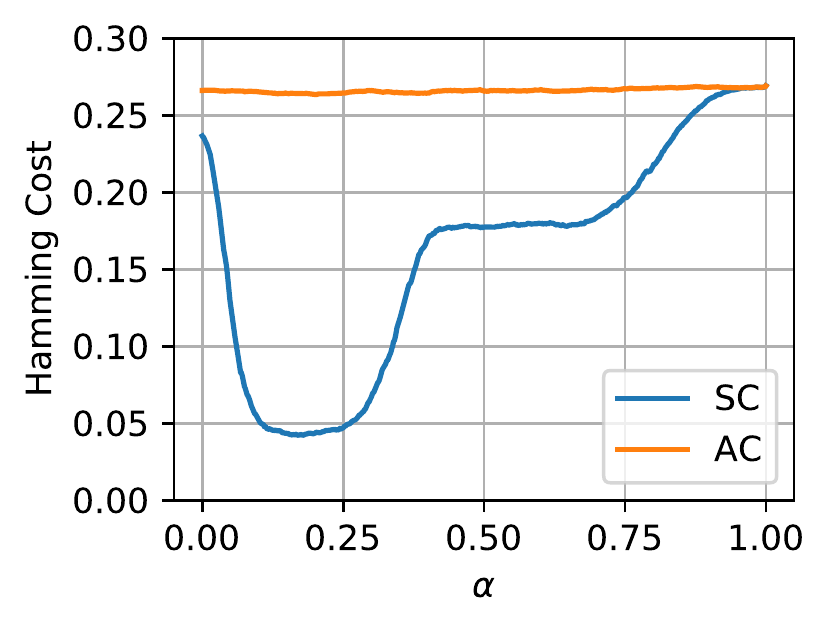}
    \caption{Rings and Disks} \label{fig:ringsAndDiscs}
  \end{subfigure}
  ~
  \begin{subfigure}[b]{0.235\textwidth}
    \includegraphics[width=\textwidth]{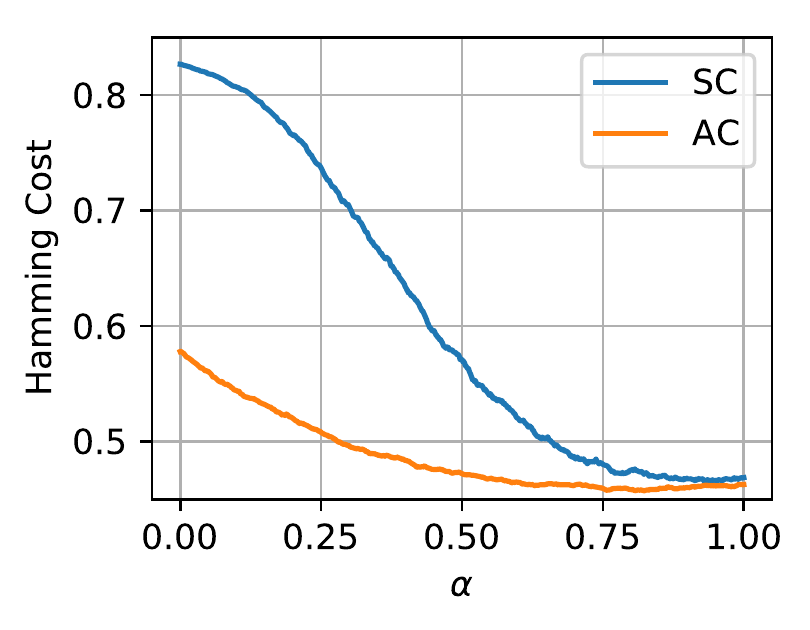}
    \caption{Omniglot} \label{fig:omniglotMerge}
  \end{subfigure}
  \caption{Empirical loss for interpolating between single and complete linkage
  (`SC' in the legend) as well as between average and complete linkage (`AC' in
  the legend) over 1000 sampled clustering instances.} \label{fig:merge}
\end{figure}

\itparagraph{Learning the Metric.} Next we consider learning the best metric for
the Omniglot, Places2, and Places2 Diverse instance distributions. Each of these
datasets is equipped with one hand-designed metric and one metric based on
neural-network embeddings. The parameter $\beta = 0$ corresponds to the
hand-designed metric, while $\beta = 1$ corresponds to the embedding.
\Cref{fig:omniglotMetric} shows the empirical loss for each parameter $\beta$
averaged over $N = 4000$ samples for each distribution. On all three
distributions the neural network embedding performs better than the
hand-designed metric, but we can achieve non-trivial performance improvements by
mixing the two metrics. On Omniglot, the optimal parameter is at $\beta = 0.514$
which improves the Hamming error by $0.091$, meaning that we correctly cluster
nearly $10\%$ more of the data. For the Places2 distribution we see an
improvement of approximately $1\%$ with the optimal parameter being $\beta =
0.88$, while for the Places2 Diverse distribution the improvement is
approximately $3.4\%$ with the optimal $\beta$ being $0.87$.

\begin{figure}[h]
  \centering
  \begin{subfigure}[t]{0.235\textwidth}
    \includegraphics[width=\textwidth]{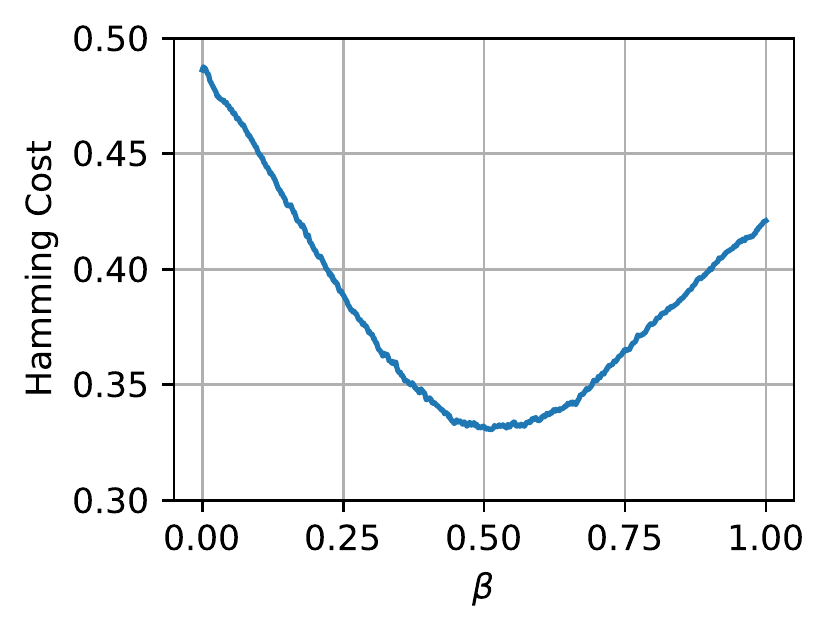}
    \caption{Omniglot}
  \end{subfigure}
  ~
  \begin{subfigure}[t]{0.235\textwidth}
    \includegraphics[width=\textwidth]{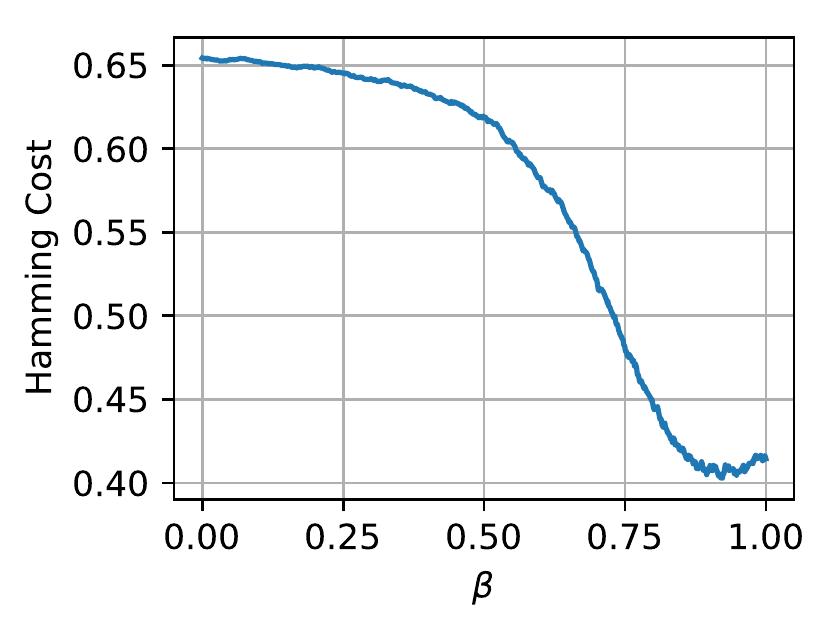}
    \caption{Places2}
  \end{subfigure}
  ~
  \begin{subfigure}[t]{0.235\textwidth}
    \includegraphics[width=\textwidth]{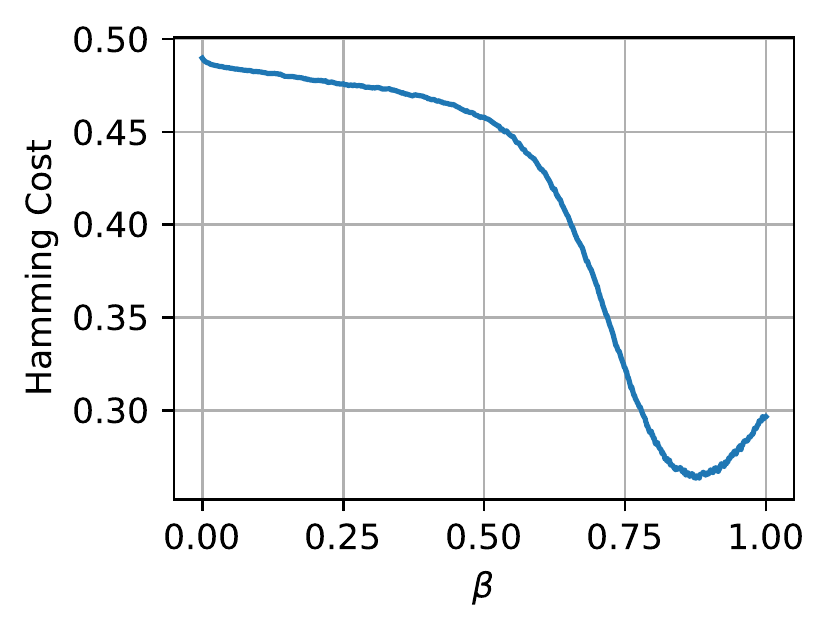}
    \caption{Places2 Diverse}
  \end{subfigure}
  \caption{Empirical loss interpolating between two distance metrics on
  Omniglot, Places2, and Places2 distributions. In each plot, $\beta = 0$
  corresponds to the hand-crafted metric and $\beta = 1$ corresponds to the
  neural network embedding.} \label{fig:omniglotMetric}
\end{figure}

\textit{Number of Discontinuities.} The efficiency of our algorithm selection
procedures stems from the fact that their running time scales with the true
number of discontinuities in each loss function, rather than a worst-case upper
bound. Of all the experiments we ran, interpolating between single interpolating
between single and complete linkage for MNIST had the most discontinuities per
loss function with an average of $362.6$ discontinuities per function. Given
that these instances have $n = 1000$ points, this leads to a speedup of roughly
$n^8 / 362.8 \approx 5.5 \times 10^{15}$ over the combinatorial algorithm that
solves for all $O(n^8)$ critical points and runs the clustering algorithm once
for each. \Cref{table:discontinuities} in \Cref{app:experiments} shows the
average number of discontinuities per loss function for all of the above
experiments.

%% file: conclusion.tex
In this work we study both the sample and algorithmic complexity of learning
linkage-based clustering algorithms with low loss for specific application
domains. We give strong bounds on the number of sample instances required from
an application domain in order to find an approximately optimal algorithm from a
rich family of algorithms that allows us to vary both the metric and merge
function used by the algorithm. We complement our sample complexity results with
efficient algorithms for finding empirically optimal algorithms for a sample of
instances. Finally, we carry out experiments on both real-world and synthetic
clustering domains demonstrating that our procedures can often find algorithms
that significantly outperform standard linkage-based clustering algorithms.

%% file: joint-family-appendix.tex
We begin by providing complete proofs for the piecewise structural Lemmas from
the main body.

\lemDistanceOrdering*
\begin{proof}
  Let $S$ be any clustering instance and fix points $a, b, a', b' \in S$. For
  any parameter $\vec{\beta} \in \Delta_\numMetrics$, by definition of
  $d_{\vec{\beta}}$, we have that
  \[
    d_{\vec{\beta}}(a, b) \leq d_{\vec{\beta}}(a', b')
    \iff
    \sum_{i=1}^\numMetrics \beta_i d_i(a,b) \leq \sum_{i=1}^\numMetrics \beta_i d_i(a',b')
    \iff
    \sum_{i=1}^\numMetrics \beta_i (d_i(a,b) - d_i(a',b')) \leq 0.
  \]
  Define the linear function $h_{a,b,a',b'}(\vec{\beta}) =
  \sum_{i=1}^\numMetrics \beta_i (d_i(a,b) - d_i(a',b'))$. Then we have that
  $d_{\vec{\beta}}(a, b) \leq d_{\vec{\beta}}(a', b')$ if
  $h_{a,b,a',b'}(\vec{\beta}) \leq 0$ and $d_{\vec{\beta}}(a, b) >
  d_{\vec{\beta}}(a', b')$ if $h_{a,b,a',b'}(\vec{\beta}) > 0$.

  Let $\cH = \{h_{a,b,a',b'} \mid a,b,a',b' \in S\}$ be the collection of all
  such linear functions collected over all possible subsets of $4$ points in
  $S$. Now suppose that $\vec{\beta}$ and $\vec{\beta}'$ belong to the same
  region in the sign-pattern partition induced by $\cH$. For any points $a, b,
  a', b' \in S$, we are guaranteed that $\sign(h_{a,b,a',b'}(\vec{\beta})) =
  \sign(h_{a,b,a',b'}(\vec{\beta}'))$, which by the above arguments imply that
  $d_{\vec{\beta}}(a,b) \leq d_{\vec{\beta}}(a',b')$ iff $d_{\vec{\beta}'}(a,b)
  \leq d_{\vec{\beta}'}(a',b')$, as required.
\end{proof}

\lemQuadraticPartition*
\begin{proof}
  From \Cref{lem:distanceOrdering}, we know we can find a set $\cH$ of
  $O(|S|^4)$ linear functions defined on $\reals^\numMetrics$ that induce a
  sign-pattern partition of the $\vec{\beta}$ parameter space
  $\Delta_\numMetrics \subset \reals^{\numMetrics}$ into regions where the
  ordering over pairs of points according to the $d_{\vec{\beta}}$ distance is
  constant.


  Now let $\region \subset \Delta_\numMetrics$ be any region of the sign-pattern
  partition of $\Delta_\numMetrics$ induced by $\cH$. From
  \Cref{lem:distanceOrdering}, we know that for all parameters $\vec{\beta} \in
  \region$, the ordering over pairs of points in $S$ according to
  $d_{\vec{\beta}}$ is fixed. For any 2-point-based merge function, the pair of
  points used to measure the distance between a pair of clusters depends only on
  the ordering of pairs of points according to distance. Therefore, since $D_1,
  \dots, D_\numMerges$ are all 2-point-based, we know that for any pair of
  clusters $(A,B)$ and each merge function index $i \in [\numMerges]$, there
  exists a pair of points $(a_i, b_i) \in A \times B$ such that $D_i(A, B;
  d_{\vec{\beta}}) = d_{\vec \beta}(a_i, b_i)$ for all $\vec{\beta} \in
  \region$. In other words, all of the merge functions measure distances between
  $A$ and $B$ using a fixed pair of points for all values of the metric
  parameter $\vec{\beta}$ in the region $\region$. Similarly, let $A', B'
  \subset S$ be any other pair of clusters and $(a_i', b_i') \in A' \times B'$
  be the pairs of points defining $D_i(A', B'; d_{\vec{\beta}})$ for each $i \in
  [\numMerges]$. Then for all $\vec{\beta} \in \region$, we have that
  \begin{align*}
  D_{\vec{\alpha}}(A, B; d_{\vec{\beta}}) \leq D_{\vec{\alpha}}(A', B'; d_{\vec{\beta}})
  &\iff \sum_{i=1}^\numMerges \alpha_i D_i(A, B; d_{\vec{\beta}}) \leq \sum_{i=1}^\numMerges \alpha_i D_i(A, B; d_{\vec{\beta}}) \\
  &\iff \sum_{i=1}^\numMerges  \alpha_i \sum_{j=1}^\numMetrics \beta_j d_j(a_i, b_i) \leq \sum_{i=1}^\numMerges  \alpha_i \sum_{j=1}^\numMetrics \beta_j d_j(a_i', b_i') \\
  &\iff \sum_{i=1}^\numMerges \sum_{j=1}^\numMetrics \alpha_i \beta_j \bigl(d_j(a_i, b_i) - d_j(a_i', b_i')\bigr) \leq 0.
  \end{align*}
  Now define the quadratic function
  \begin{equation}
    q_{A,B,A',B'}(\vec{\alpha}, \vec{\beta}) = \sum_{i=1}^\numMerges \sum_{j=1}^\numMetrics \alpha_i \beta_j \bigl(d_j(a_i, b_i) - d_j(a_i', b_i')\bigr). \label{eq:clusterQuad}
  \end{equation}
  For all $\vec{\beta} \in \region$, we are guaranteed that $D_{\vec{\alpha}}(A,
  B; d_{\vec{\beta}}) \leq D_{\vec{\alpha}}(A', B'; d_{\vec{\beta}})$ if and
  only if $q_{A,B,A',B'}(\vec{\alpha}, \vec{\beta}) \leq 0$. Notice that the
  coefficients of $q_{A,B,A',B'}$ only depend on $4\numMerges$ points in $S$,
  which implies that if we collect these quadratic functions over all quadruples
  of clusters $A, B, A', B' \subset S$, we will only obtain
  $O(|S|^{4\numMerges})$ different quadratic functions. These
  $O(|S|^{4\numMerges})$ functions induce a sign-pattern partition of
  $\Delta_\numMerges \times \region$ for which the desired conclusion holds.
  Next, observe that the coefficients in the quadratic functions defined above
  do not depend on the region $\region$ we started with. It follows that the
  same set of $O(|S|^{4 \numMerges})$ quadratic functions partition any other
  region $\region'$ in the sign-pattern partition induced by $\cH$ so that the
  claim holds on $\Delta_{\numMerges} \times \region'$.

  Now let $\cQ$ contain the linear functions in $\cH$ (viewed as quadratic
  functions over $\reals^{\numMerges + \numMetrics}$ by placing a zero
  coefficient on all quadratic terms and terms depending on $\vec{\alpha}$),
  together with the $O(|S|^{4\numMerges})$ quadratic functions defined above.
  Then we have that $|\cQ| = O(|S|^4 + |S|^{4\numMerges}) =
  O(|S|^{4\numMerges})$. Now suppose that $(\vec{\alpha},\vec{\beta})$ and
  $(\vec{\alpha}',\vec{\beta}')$ belong to the same region of the sign-pattern
  partition of $\Delta_\numMerges \times \Delta_\numMetrics \subset
  \reals^{\numMerges + \numMetrics}$ induced by the quadratic functions $\cQ$.
  Since $\cQ$ contains $\cH$, this implies that $\vec{\beta}$ and $\vec{\beta'}$
  belong to the same region $\region$ in the sign-pattern partition induced by
  $\cH$. Moreover, since $\cQ$ contains all the quadratic functions defined in
  \eqref{eq:clusterQuad}, it follows that $D_{\vec{\alpha}}(A, B;
  d_{\vec{\beta}}) \leq D_{\vec{\alpha}}(A', B'; d_{\vec{\beta}})$ if and only
  if $D_{\vec{\alpha}'}(A, B; d_{\vec{\beta}'}) \leq D_{\vec{\alpha}'}(A', B';
  d_{\vec{\beta}'})$, as required.
\end{proof}

Next, we prove \Cref{thm:sampleComplexity}.

\thmSampleComplexity*
\begin{proof}[Proof of \Cref{thm:sampleComplexity}]
  Define the class of loss functions $\cL = \{\ell_{\vec{\alpha},
  \vec{\beta}}(S, \target) = \ell(\cA_{\vec{\alpha}, \vec{\beta}}(S), \target)
  \mid (\vec{\alpha}, \vec{\beta}) \in \Delta_\numMerges \times
  \Delta_\numMetrics \}$. Let $p = \numMerges + \numMetrics$ be the dimension of
  the joint parameter space. If we can bound the pseudo-dimension of $\cL$ by
  $O(p^2(\log(p) + \numMerges \log(n))$, then the result follows immediately
  from standard pseudo-dimension based sample complexity guarantees
  \citep{Pollard84:Convergence}.

  \citet{Balcan19:General} show how to bound the pseudo-dimension of any
  function class $\cL$ when the class of \emph{dual functions} $\cL^*$ is
  piecewise structured. For each clustering instance $S$ with target clustering
  $\target$, there is one dual function $\ell_{S, \target} : \Delta_\numMerges
  \times \Delta_\numMetrics \to \reals$ defined by $\ell_{S,
  \target}(\vec{\alpha}, \vec{\beta}) = \ell_{\vec{\alpha}, \vec{\beta}}(S,
  \target)$. The key structural property we proved in
  \Cref{lem:quadraticPartition} guarantees that each dual function is piecewise
  constant, and the constant partition is the sign-pattern partition induced by
  $O(n^\numMerges)$ quadratic functions. Applying Theorem 3.1 of
  \citet{Balcan19:General}, we have that the $\pdim(\cL) = O(V \log(V) + V
  \numMerges \log n)$, where $V$ is the VC-dimension of the dual class to
  quadratic separators defined on $\reals^{p}$. The dual class consists of
  linear functions defined over $\reals^{p^2 + p + 1}$, and therefore its
  VC-dimension is bounded by $V = O(p^2)$. It follows that $\pdim(\cL) = O(p^2
  \log(p) + p^2 \numMerges \log(n))$, as required.
\end{proof}

%% file: erm-algorithms-appendix.tex
\subsection{Learning the Merge Function}

In this section we provide details for learning the best combination of two
merge functions. We also give detailed pseudocode for our sweepline algorithm
for finding the children of a node in the execution tree (see
\Cref{alg:findMerges}) and for the complete algorithm (see
\Cref{alg:depthFirst}).

\begin{algorithm}
  \textbf{Input:} Set of clusters $C_1, \dots, C_m$, merge functions $\DZero, \DOne$, parameter interval $[\alo, \ahi)$.
  \begin{enumerate}[nosep, leftmargin=*]
    \item Let $\mathcal{M} = \emptyset$ be the initially empty set of possible merges.
    \item Let $\mathcal{I} = \emptyset$ be the initially empty set of parameter intervals.
    \item Let $\alpha = \alo$.
    \item While $\alpha < \ahi$:
    \begin{enumerate}[nosep, leftmargin=*]
      \item Let $C_i, C_j$ be the pair of clusters minimizing $(1-\alpha)\cdot \DZero(C_i, C_j) + \alpha \cdot \DOne(C_i, C_j)$.
      \item For each $k, l \in [m]$, let $c_{kl} = \Delta_0 / (\Delta_0 - \Delta_1)$, where $\Delta_p = \mergep{p}(C_i, C_j) - \mergep{p}(C_k, C_l)$ for $p \in \{0,1\}$.
      \item Let $c = \min \bigl( \{ c_{kl} \,|\, c_{kl} > \alpha\} \cup \{\ahi\} \bigr)$ .
      \item Add merge $(C_i, C_j)$ to $\mathcal{M}$ and $[\alpha, c)$ to $\mathcal{I}$.
      \item Set $\alpha = c$.
    \end{enumerate}
    \item Return $\mathcal{M}$ and $\mathcal{I}$.
  \end{enumerate}
  \caption{Find all merges for $\mergeFamily(\DZero,\DOne)$ }
  \label{alg:findMerges}
\end{algorithm}

\begin{algorithm}
\textbf{Input:} Point set $x_1, \dots, x_n$, cluster distance functions $d_1$ and $d_2$.
\begin{enumerate}[nosep, leftmargin=*]
\item Let $r$ be the root node of the execution tree with $r.\clusters = \{ (x_1), \dots, (x_n) \}$ and $r.I = [0,1]$.
\item Let $s$ be a stack of execution tree nodes, initially containing the root $r$.
\item Let $\mathcal{T} = \emptyset$ be the initially empty set of possible cluster trees.
\item Let $\mathcal{I} = \emptyset$ be the initially empty set of intervals.
\item While the stack $s$ is not empty:
  \begin{enumerate}[nosep, leftmargin=*]
  \item Pop execution tree node $e$ off stack $s$.
  \item If $e.\clusters$ has a single cluster, add $e.\clusters$ to $\mathcal{T}$ and $e.I$ to $\mathcal{I}$.
  \item Otherwise, for each merge $(C_i, C_j)$ and interval $I_c$ returned by \Cref{alg:findMerges} run on $e.\clusters$ and $e.I$:
    \begin{enumerate}
    \item Let $c$ be a new node with state given by $e.\clusters$ after merging $C_i$ and $C_j$ and $c.I = I_c$.
    \item Push $c$ onto the stack $s$.
    \end{enumerate}
  \end{enumerate}
  \item Return $\mathcal{T}$ and $\mathcal{I}$.
\end{enumerate}
\caption{Depth-first Enumeration of $\alpha$-linkage Execution Tree}
\label{alg:depthFirst}
\end{algorithm}

\lemMergeExecutionTree*
\begin{proof}
  The proof is by induction on the depth $t$. The base case is for depth $t =
  0$, in which case we can use a single node whose interval is $[0,1]$. Since
  all algorithms in the family start with an empty-sequence of merges, this
  satisfies the execution tree property.

  Now suppose that there is a tree of depth $t$ with the execution tree
  property. If $t = |S|-1$ then we are finished, since the algorithms in
  $\mergeFamily(\DZero, \DOne)$ make exactly $|S|-1$ merges. Otherwise, consider
  any leaf node $v$ of the depth $t$ tree with parameter interval $I_v$. It is
  sufficient to show that we can partition $I_v$ into subintervals such that for
  $\alpha$ in each subinterval the next merge performed is constant. By the
  inductive hypothesis, we know that the first $t$ merges made by
  $\mergeAlg{\alpha}$ are the same for all $\alpha \in I_v$. After performing
  these merges, the algorithm will have arrived at some set of clusters $C_1,
  \dots, C_m$ with $m = |S| - t$.  For each pair of clusters $C_i$ and $C_j$,
  the distance $\Dalpha(C_i, C_j) = (1-\alpha)\DZero(C_i, C_j) + \alpha
  \DOne(C_i, C_j)$ is a linear function of the parameter $\alpha$. Therefore,
  for any clusters $C_i, C_j, C_k$, and $C_l$, the algorithm will prefer to
  merge $C_i$ and $C_j$ over $C_j$ and $C_k$ for a (possibly empty) sub-interval
  of $I_v$, corresponding to the values of $\alpha \in I_v$ where $\Dalpha(C_i,
  C_j) < \Dalpha(C_k, C_l)$. For any fixed pair of clusters $C_i$ and $C_j$,
  taking the intersection of these intervals over all other pairs $C_j$ and
  $C_k$ guarantees that clusters $C_i$ and $C_j$ will be merged exactly for
  parameter values in some subinterval of $I_v$. For each merge with a non-empty
  parameter interval, we can introduce a child node of $v$ labeled by that
  parameter interval. These children partition $I_v$ into intervals where the
  next merge is constant, as required.
\end{proof}

\lemFindMergesRuntime*
\begin{proof}
  The loop in step 4 of \Cref{alg:findMerges} runs once for each possible merge,
  giving a total of $M$ iterations. Each iteration finds the closest pair of
  clusters according to $\Dalpha$ using $O(m^2)$ evaluations of the merge
  functions $\DZero$ and $\DOne$. Calculating the critical parameter value $c$
  involves solving $O(m^2)$ linear equations whose coefficients are determined
  by four evaluations of $\DZero$ and $\DOne$. It follows that the cost of each
  iteration is $O(m^2 K)$, where $K$ is the cost of evaluating $\DZero$ and
  $\DOne$, and the overall running time is $O(Mm^2 K)$.
\end{proof}

\thmDepthFirstRuntime*
\begin{proof}
  Fix any node $v$ in the execution tree with $m$ clusters $C_1, \dots, C_m$ and
  $M$ outgoing edges (i.e., $M$ possible merges from the state represented by
  $v$). We run \Cref{alg:findMerges} to determine the children of $v$, which by
  \Cref{lem:findMergesRuntime} costs $O(M n^2 K)$, since $m \leq n$. Summing
  over all non-leaves of the execution tree, the total cost is $O(E n^2 K)$. In
  addition to computing the children of a given node, we need to construct the
  children nodes, but this takes constant time per child.
\end{proof}

\subsection{Learning the Metric}

In this section we provide details for learning the best combination of two
metrics. We also give detailed pseudocode for our sweepline algorithm for
finding the children of a node in the execution tree (see
\Cref{alg:findMergesMetric}) and for the complete algorithm (see
\Cref{alg:depthFirstMetric}).

\lemMetricExecutionTree*
\begin{proof}
  The proof is by induction on the depth $t$. The base case is for depth $t =
  0$, in which case we can use a single node whose interval is $[0,1]$. Since
  all algorithms in the family start with an empty-sequence of merges, this
  satisfies the execution tree property.

  Now suppose that there is a tree of depth $t$ with the execution tree
  property. If $t = |S|-1$ then we are finished, since the algorithms in
  $\metricFamily(\dZero, \dOne)$ make exactly $|S|-1$ merges. Otherwise,
  consider any leaf node $v$ of the depth $t$ tree with parameter interval
  $I_v$. It is sufficient to show that we can partition $I_v$ into subintervals
  such that for $\beta$ in each subinterval the next merge performed is
  constant. By the inductive hypothesis, we know that the first $t$ merges made
  by $\metricAlg{\beta}$ are the same for all $\beta \in I_v$. After performing
  these merges, the algorithm will have arrived at some set of clusters $C_1,
  \dots, C_m$ with $m = |S| - t$. Recall that algorithms in the family
  $\metricFamily(\dZero, \dOne)$ run complete linkage using the metric $\dbeta$.
  Complete linkage can be implemented in such a way that it only makes
  comparisons between pairwise point distances (i.e., is $\dbeta(x,x')$ larger
  or smaller than $\dbeta(y,y')$?). To see this, for any pair of clusters, we
  can find the farthest pair of points between them using only distance
  comparisons. And, once we have the farthest pair of points between all pairs
  of clusters, we can find the pair of clusters to merge by again making only
  pairwise comparisons. It follows that if two parameters $\beta$ and $\beta'$
  have the same outcome for all pairwise distance comparisons, then the next
  merge to be performed must be the same. We use this observation to partition
  the interval $I_v$ into subintervals where the next merge is constant. For any
  pair of points $x, x' \in S$, the distance $\dbeta(x,x') = (1-\beta)
  \dZero(x,x') + \beta \dOne(x,x')$ is a linear function of the parameter
  $\beta$. Therefore, for any points $x,x', y, y' \in S$, there is at most one
  critical parameter value where the relative order of $\dbeta(x,x')$ and
  $\dbeta(y,y')$ changes. Between these $O(|S|^4)$ critical parameter values,
  the ordering on all pairwise merges is constant, and the next merge performed
  by the algorithm will also be constant. Therefore, there must exist a
  partitioning of $I_v$ into at most $O(|S|^4)$ sub-intervals such that the next
  merge is constant on each interval. We let the children of $v$ correspond to
  the coarsest such partition.
\end{proof}

\lemFindMergesRuntimeMetric*
\begin{proof}
  The loop in step 4 of \Cref{alg:findMergesMetric} runs once for each possible
  merge, giving a total of $M$ iterations. Each iteration finds the merge
  performed by complete linkage using the $\dbeta$ metric, which takes $O(n^2)$
  time, and then solves $O(n^2)$ linear equations to determine the largest value
  of $\beta'$ such that the same merge is performed. It follows that the cost of
  each iteration is $O(n^2)$, leading to an overall running time of $O(Mn^2)$.
  Note, we assume that the pairwise distances $\dbeta(x,x')$ can be evaluated in
  constant time. This can always be achieved by precomputing two $n \times n$
  distance matrices for the base metrics $\dZero$ and $\dOne$, respectively.
\end{proof}

\begin{algorithm}
  \textbf{Input:} Set of clusters $C_1, \dots, C_m$, metrics $\dZero, \dOne$, parameter interval $[\blo, \bhi)$.
  \begin{enumerate}[nosep, leftmargin=*]
    \item Let $\mathcal{M} = \emptyset$ be the initially empty set of possible merges.
    \item Let $\mathcal{I} = \emptyset$ be the initially empty set of parameter intervals.
    \item Let $\beta = \blo$.
    \item While $\beta < \bhi$:
    \begin{enumerate}[nosep, leftmargin=*]
      \item Let $C_i, C_j$ be the pair of clusters minimizing $\max_{a \in A, b \in B} \dbeta(a, b)$.
      \item Let $x \in C_i$ and $x' \in C_j$ be the farthest points between $C_i$ and $C_j$.
      \item For all pairs of points $y$ and $y'$ belonging to different clusters, let $c_{yy'} = \Delta_0 / (\Delta_0 - \Delta_1)$ where $\Delta_p = \metricp{p}(y,y') - \metricp{p}(x,x')$ for $p \in \{0,1\}$.
      \item Let $c = \min \bigl( \{ c_{yy'} \,|\, c_{yy'} > \beta \} \cup \{\bhi\} \bigr)$ .
      \item Add merge $(C_i, C_j)$ to $\mathcal{M}$ and $[\beta, c)$ to $\mathcal{I}$.
      \item Set $\beta = c$.
    \end{enumerate}
    \item Return $\mathcal{M}$ and $\mathcal{I}$.
  \end{enumerate}
  \caption{Find all merges for $\metricFamily(\dZero,\dOne)$ }
  \label{alg:findMergesMetric}
\end{algorithm}

\thmDepthFirstRuntimeMetric*
\begin{proof}
  Fix any node $v$ in the execution tree with $m$ clusters $C_1, \dots, C_m$ and
  $M$ outgoing edges (i.e., $M$ possible merges from the state represented by
  $v$). We run \Cref{alg:findMergesMetric} to determine the children of $v$,
  which by \Cref{lem:findMergesRuntimeMetric} costs $O(M n^2)$. Summing over all
  non-leaves of the execution tree, the total cost is $O(E n^2)$.
\end{proof}

\begin{algorithm}
\textbf{Input:} Point set $x_1, \dots, x_n$, cluster distance functions $d_1$ and $d_2$.
\begin{enumerate}[nosep, leftmargin=*]
\item Let $r$ be the root node of the execution tree with $r.\clusters = \{ (x_1), \dots, (x_n) \}$ and $r.I = [0,1]$.
\item Let $s$ be a stack of execution tree nodes, initially containing the root $r$.
\item Let $\mathcal{T} = \emptyset$ be the initially empty set of possible cluster trees.
\item Let $\mathcal{I} = \emptyset$ be the initially empty set of intervals.
\item While the stack $s$ is not empty:
  \begin{enumerate}[nosep, leftmargin=*]
  \item Pop execution tree node $e$ off stack $s$.
  \item If $e.\clusters$ has a single cluster, add $e.\clusters$ to $\mathcal{T}$ and $e.I$ to $\mathcal{I}$.
  \item Otherwise, for each merge $(C_i, C_j)$ and interval $I_c$ returned by \Cref{alg:findMergesMetric} run on $e.\clusters$ and $e.I$:
    \begin{enumerate}
    \item Let $c$ be a new node with state given by $e.\clusters$ after merging $C_i$ and $C_j$ and $c.I = I_c$.
    \item Push $c$ onto the stack $s$.
    \end{enumerate}
  \end{enumerate}
  \item Return $\mathcal{T}$ and $\mathcal{I}$.
\end{enumerate}
\caption{Depth-first Enumeration of $\beta$-linkage Execution Tree}
\label{alg:depthFirstMetric}
\end{algorithm}

%% file: experiments-new-appendix.tex
\paragraph{Clustering distributions.}

\noindent\textit{MNIST Subsets.} Our first distribution over clustering tasks
corresponds to clustering subsets of the MNIST dataset \citep{LeCun98:MNIST},
which contains 80,000 hand-written examples of the digits $0$ through $9$. We
generate a random clustering instance from the MNIST data as follows: first, we
select $k = 5$ digits from $\{0, \dots, 9\}$ at random, then we randomly select
$200$ examples belonging to each of the selected digits, giving a total of $n =
1000$ images. The target clustering for this instance is given by the
ground-truth digit labels. We measure distances between any pair of digits in
terms of the the Euclidean distance between their images represented as vectors
of pixel intensities.

\itparagraph{CIFAR-10 Subsets.} We also consider a distribution
over clustering tasks that corresponds to clustering subsets of the CIFAR-10
dataset \citep{Krizhevsky09:CIFAR}. This dataset contains $6000$ images of each
of the following classes: airplane, automobile, bird, cat, deer, dog, frog,
horse, ship, and truck. Each example is a $32 \times 32$ color image with 3
color channels. We pre-process the data to obtain neural-network feature
representations for each example. We include 50 randomly rotated and cropped
versions of each example and obtain feature representations from layer `in4d' of
a pre-trained Google inception network.
This gives a $144$-dimensional feature representation for each of the $3000000$
examples (50 randomly rotated copies of the 6000 examples for each of the 10
classes). We generate clustering tasks from CIFAR-10 as follows: first, select $k
= 5$ classes at random, then choose $50$ examples belonging to each of the
selected classes, giving a total of $n = 250$ images. The target clustering for
this instance is given by the ground-truth class labels. We measure distance
between any pair of images as the distance between their feature embeddings.

\itparagraph{Omniglot Subsets.} Next, we consider a distribution over clustering
tasks corresponding to clustering subsets of the Omniglot dataset
\citep{Lake15:Omniglot}. The Omniglot dataset consists of written characters
from 50 different alphabets with a total of 1623 different characters. The
dataset includes 20 examples of each character, leading to a total of 32,460
examples. We generate a random clustering instance from the Omniglot data as
follows: first, we choose one of the alphabets at random. Next, we choose $k$
uniformly in $\{5, \dots, 10\}$ and choose $k$ random characters from that
alphabet. The clustering instance includes $20k$ examples and the target
clustering is given by the ground-truth character labels.

We use two different distance metrics on the Omniglot dataset. First, we use the
cosine distance between neural network feature embeddings. The neural network
was trained to perform digit classification on MNIST. Second, each example has
both an image of the written character, as well as the stroke trajectory (i.e.,
a time series of $(x,y)$ coordinates of the tip of the pen when the character
was written). We also use the following distance defined in terms of the
strokes: Given two trajectories $s = (x_t, y_t)_{t=1}^T$ and $s' = (x'_t,
y'_t)_{t=1}^T$, we define the distance between them by $ d(s,s') = \frac{1}{T +
T'} \left( \sum_{t=1}^T d\bigl((x_t, y_t), s'\bigr) + \sum_{t=1}^{T'}
d\bigl((x'_t, y'_t), s\bigr) \right), $ where $d\bigl((x_t, y_t), s'\bigr)$
denotes the Euclidean distance from the point $(x_t, y_t)$ to the closest point
in $s'$. This is the average distance from any point from either trajectory to
the nearest point on the other trajectory. This hand-designed metric provides a
complementary notion of distance to the neural network feature embeddings.

\itparagraph{Places2 Subsets.} The Places2 dataset consists of images of 365
different place categories, including ``volcano'', ``gift shop'', and
``farm''~\citep{Zhou17:Places}. To generate a clustering instance from the
places data, we choose $k$ randomly from $\{5, \dots, 10\}$, choose $k$ random
place categories, and then select $20$ random examples from each chosen
category.  We restrict ourselves to the first 1000 images from each class.

We use two metrics for this data distribution. First, we use cosine distances
between feature embeddings generated by a VGG16 network pre-trained on imagenet.
In particular, we use the activations just before the fully connected layers,
but after the max-pooling is performed, so that we have $512$-dimensional
feature vectors. Second, we compute color histograms in HSV space for each image
and use the cosine distance between the histograms. In more detail, we partition
the hue space into $8$ bins, the saturation space into $2$ bins, and the value
space into $4$ bins, resulting in a 64-dimensional histogram counting how
frequently each quantized color appears in the image. Two images are close under
this metric if they contain similar colors.

\itparagraph{Places2 Diverse Subsets.} We also construct an instance
distribution from a subset of the Places2 classes which have diverse color
histograms. We expect the color histogram metric to perform better on this
distribution. To generate a clustering instance, we pick $k = 4$ classes from
aquarium, discotheque, highway, iceberg, kitchen, lawn, stage-indoor, underwater
ocean deep, volcano, and water tower. We include $50$ randomly sampled images
from each chosen class, leading to a total of $n = 200$ points per instance.

\itparagraph{Synthetic Rings and Disks.} We consider a two
dimensional synthetic distribution where each clustering instance has 4
clusters, where two are ring-shaped and two are disk-shaped. To generate each
instance we sample 100 points uniformly at random from each ring or disk. The
two rings have radiuses $0.4$ and $0.8$, respectively, and are both centered at
the origin. The two disks have radius $0.4$ and are centered at $(1.5, 0.4)$ and
$(1.5, -0.4)$, respectively. For this data, we measure distances between points
in terms of the Euclidean distance between them.

\paragraph{Average Number of Discontinuities}. Next we report the average number
of discontinuities in the loss function for a clustering instance sampled from
each of the distributions described above for each of the learning tasks we
consider. In all cases, the average number of discontinuities is many orders of
magnitude smaller than the upper bounds. The metric learning problems tend to
have more discontinuities than learning the best merge function. Surprisingly,
even though our only worst-case bound on the number of discontinuities when
interpolating between average and complete linkage is exponential in $n$, the
empirical number of discontinuities is always smaller than for interpolating
between single and complete linkage. The results are shown in
\Cref{table:discontinuities}.

\begin{table}[h!]
\centering
\begin{tabular}{c  c  c  c}
 \hline
 Distribution & Task & max $n$ & Average \# Discontinuities \\
 \hline\hline
 Omniglot & SC & 200 & 59.4 \\
 Omniglot & AC & 200 & 33.9 \\
 Omniglot & metric & 200 & 201.1 \\
 \hline
 MNIST & SC & 1000 & 362.6 \\
 MNIST & AC & 1000 & 282.0 \\
 \hline
 Rings and Disks & SC & 400 & 29.0 \\
 Rings and Disks & AC & 400 & 18.3 \\
 \hline
 CIFAR-10 & SC & 250 & 103.2 \\
 CIFAR-10 & AC & 250 & 66.2 \\
 \hline
 Places2 & metric & 200 & 241.0 \\
 Places2 Diverse & metric & 200 & 269.6
\end{tabular}
\caption{Table of average number of discontinuities for a piecewise constant loss function sampled from each distribution and learning task. Task `SC' corresponds to interpolating between single and complete linkage, `AC` is interpolating between average and complete linkage, and `metric` is interpolating between two base metrics. The column labeled ``max $n$'' is an upper bound on the size of each clustering instance.}
\label{table:discontinuities}
\end{table}